\documentclass[10pt]{article}

\usepackage{fullpage}
\usepackage[T1]{fontenc}

\usepackage{amsmath,amsthm,amsfonts,amssymb}
\usepackage[colorlinks=true,linkcolor=blue,citecolor=green!50!black]{hyperref}
\usepackage[nameinlink,capitalise]{cleveref}
\usepackage{cite}

\usepackage{authblk}

\usepackage{caption}
\usepackage{enumitem}

\usepackage{import}

\usepackage{tikz}
\usetikzlibrary{calc,angles}
\usetikzlibrary{plotmarks}
\usetikzlibrary{shapes,fit}
\usetikzlibrary{decorations.markings}
\usepackage{pgfplots}
\usepackage{pgfplotstable}
\pgfplotsset{compat=newest}
\usepgfplotslibrary{statistics,groupplots}
\pgfdeclarelayer{back}
\pgfdeclarelayer{front}
\pgfsetlayers{back,main,front}
\usepackage{tikz-cd}

\usepackage{xcolor}

\definecolor{darkgreen}{RGB}{21,140,21}
\definecolor{darkblue}{RGB}{21,21,140}
\definecolor{sun}{RGB}{234,171,0}

\colorlet{shadecolor}{black!5}

\definecolor{pastelMagenta}{HTML}{FF48CF}
\definecolor{pastelPurple}{HTML}{8770FE}
\definecolor{pastelBlue}{HTML}{1BA1EA}
\definecolor{pastelSeaGreen}{HTML}{14B57F}
\definecolor{pastelGreen}{HTML}{3EAA0D}
\definecolor{pastelOrange}{HTML}{F97306}
\definecolor{pastelRed}{HTML}{F5615C}

\pgfplotscreateplotcyclelist{threelines}{%
{color=pastelBlue,mark=*,solid,mark options={solid}},
{color=pastelOrange,mark=triangle*,dashed,mark options={solid}},
{color=pastelGreen,mark=square*,dotted,thick,mark options={solid}},
}

\pgfplotscreateplotcyclelist{paircyclelist}{%
{color=pastelBlue,mark=o,solid,mark options={solid}},
{color=pastelBlue,mark=o,dashed,mark options={solid}},
{color=pastelOrange,mark=*,solid},
{color=pastelOrange,mark=*,dashed},
{color=pastelGreen,mark=square*,solid},
{color=pastelGreen,mark=square*,dashed},
}

\usepackage{algpseudocodex}

\DeclareMathAlphabet{\mathmybb}{U}{bbold}{m}{n}

\newcommand{\dft}[5]{#1\overset{#2}{\leftarrow}#3\overset{#4}{\rightarrow}#5}

\DeclareMathOperator{\gdet}{gdet}
\DeclareMathOperator{\wf}{WF}
\DeclareMathOperator{\Lip}{Lip}

\DeclareMathOperator*{\supp}{supp}
\DeclareMathOperator*{\vspan}{span}

\DeclareMathOperator*{\polylog}{\mathrm{polylog}}
\DeclareMathOperator*{\effrank}{\tilde{r}}

\makeatletter
\newsavebox{\@brx}
\newcommand{\llangle}[1][]{\savebox{\@brx}{\(\m@th{#1\langle}\)}%
  \mathopen{\copy\@brx\kern-0.5\wd\@brx\usebox{\@brx}}}
\newcommand{\rrangle}[1][]{\savebox{\@brx}{\(\m@th{#1\rangle}\)}%
  \mathclose{\copy\@brx\kern-0.5\wd\@brx\usebox{\@brx}}}
\makeatother

\newtheorem{theorem}{Theorem}

\newtheorem{lemma}{Lemma}

\newtheorem{proposition}{Proposition}

\crefname{conjecture}{Conjecture}{Conjectures}
\crefname{claim}{Claim}{Claims}

\theoremstyle{definition}
\newtheorem{definition}{Definition}

\newtheorem{assumption}{Assumption}
\crefname{assumption}{Assumption}{Assumptions}

\newtheorem{example}{Example}

\newtheorem{remark}{Remark}

\usepackage{mdframed}
\usepackage{fmtcount}
\newtheoremstyle{takeawaystyle}
      {\topsep}
      {\topsep}
      {\itshape}
      {}
      {\bfseries}
      {:}
      { }
      {\Ordinalstringnum{#2} \thmname{#1}\thmnote{. #3}}
\theoremstyle{takeawaystyle}
\newmdtheoremenv[innertopmargin=0pt,leftmargin=2cm,rightmargin=2cm,innerleftmargin=0.5cm,innerrightmargin=0.5cm]{takeaway}{Takeaway}

\crefformat{takeaway}{#2{\Ordinalstringnum{#1}#3 Takeaway}}
\crefrangeformat{takeaway}{Takeaways~(#3#1#4) to~(#5#2#6)}
\crefmultiformat{takeaway}{#2\Ordinalstringnum{#1}#3}{ and~#2\Ordinalstringnum{#1}#3}{, #2\Ordinalstringnum{#1}#3}{ and~#2\Ordinalstringnum{#1}#3 Takeaways}

\begin{document}

\title{Rates and architectures for learning \\ geometrically non-trivial operators} 

\author[a]{T.~Mitchell~Roddenberry}
\author[b]{Leo Tzou}
\author[c]{Ivan Dokmani\'{c}}
\author[d]{\authorcr Maarten V. de Hoop}
\author[a]{Richard G. Baraniuk}

\affil[a]{Department of Electrical and Computer Engineering, Rice University, Houston, TX, USA}
\affil[b]{School of Mathematics and Statistics, University of Melbourne, Melbourne, Australia}
\affil[c]{Department of Mathematics and Computer Science, University of Basel, Basel, Switzerland}
\affil[d]{Department of Computational Mathematics and Operations Research, Rice University, Houston, TX, USA}

\date{\vspace{-5ex}}

\maketitle

\begin{abstract}
Deep learning methods have proven capable of recovering operators between high-dimensional spaces, such as solution maps of PDEs and similar objects in mathematical physics, from very few training samples.
This phenomenon of data-efficiency has been proven for certain classes of elliptic operators with simple geometry, \textit{i.e.}, operators that do not change the domain of the function or propagate singularities.
However, scientific machine learning is commonly used for problems that \emph{do} involve the propagation of singularities in \textit{a priori} unknown ways, such as waves, advection, and fluid dynamics.
In light of this, we expand the learning theory to include double fibration transforms--geometric integral operators that include generalized Radon and geodesic ray transforms.
We prove that this class of operators does not suffer from the curse of dimensionality: the error decays superalgebraically, that is, faster than any fixed power of the reciprocal of the number of training samples.
Furthermore, we investigate architectures that explicitly encode the geometry of these transforms, demonstrating that an architecture reminiscent of cross-attention based on levelset methods yields a parameterization that is universal, stable, and learns double fibration transforms from very few training examples.
Our results contribute to a rapidly-growing line of theoretical work on learning operators for scientific machine learning.
\end{abstract}

\medskip

\noindent Learning integral kernel operators, often stemming from partial differential equations (PDEs), is a staple of scientific machine learning~\cite{yu2024learning,raissi2020hidden,karniadakis2021physics}.
It is a key component of neural operators that approximate nonlinear maps between function spaces~\cite{li2020fno,kovachki2021universal,li2023fourier}.
Recent advances promise to characterize phenomena for which mathematical descriptions are currently out of reach, and to greatly accelerate numerical solutions of partial differential equations and inverse problems~\cite{kothari2020learning,yu2024learning,wang2021learning,de2019deep,fan2020solving,E2018deepRitz,Wang2021DyAd,Karthik2021physics,cuomo2022scientific,schaeffer2017learning}. 
The inputs and outputs of these operators are high or infinite-dimensional spatio(temporal) fields~\cite{boulle2023,bhattacharya2020model,rao2022discovering,long2019pde}, so
learning is in general subject to the curse of dimensionality; learning a $d$-variate scalar Lipschitz function to precision $\epsilon$ requires $\Omega(\epsilon^{-d})$ samples~\cite[Theorem~15]{luxburg2004}. 
It is thus important to establish that classes of interest are not cursed. 
Indeed, various operators arising from physics can be efficiently learned using relatively few training samples~\cite{lu2021learning,boulle2023,chen2023deep}.

Recent work explains this data-efficiency phenomenon for an important class of operators, namely source-to-solution maps for elliptic PDEs~\cite{boulle2023}.
Although the main result in~\cite{boulle2023} is established using carefully constructed example data tailored to an \textit{a priori} known multiscale structure, it sheds light on how deep learning methods whose implicit bias favors such operators, such as GreenLearning~\cite{boulle2022data}, can achieve low testing error from very few training samples.

However, this does not resolve the important question of learning operators with \emph{non-trivial geometry}, because the solution operators for elliptic PDEs are pseudolocal.
Given that methods of scientific machine learning are being applied to problems arising in electrostatics~\cite{fan2020solving}, groundwater hydraulics~\cite{prabhat_nature_2019}, and tomography~\cite{monard2014,FGGN,khorashadizadeh2024glimpse},
we are motivated to establish learning guarantees for operators with non-trivial geometry of this type. 
We study double fibration transforms in particular, and prove that learning this class of operators does not suffer from the curse of dimensionality.
Our sampling rates use training examples drawn from generic random fields, with no hierarchical structure or ordering, and thus meshes naturally with standard supervised learning theory. 
The key property underlying this result is the sparse structure of double fibrations when represented using time-frequency atoms, allowing us to leverage ideas from compressive sensing to establish upper bounds on learning rates.

Following this, we develop a parameterization of these operators that represents the geometric relation \emph{implicitly}~(\cref{defn:levelset-integral-kernel}).
We prove that this parameterization asymptotically represents the appropriate class of operators~(\cref{lemma:gauss-approximation}), can be factorized for computational efficiency without any loss of expressivity~(\cref{lemma:universal}), and is robust to discretization of the domain~(\cref{lemma:discrete-kernel}).
Even when the dimension of the geometric relation is unknown, we demonstrate that it can be estimated from the learned parameterization by examining the Jacobian of the learned map over incident points.
We show that this architecture can approximate double fibration transforms well from few training samples.
Additionally, through an illustrative example from Riemannian geometry, we show that it is amenable to \textit{post hoc} analysis to determine the metric underlying a geodesic ray transform.

\section{Linear Operator Learning}

Given two potentially infinite-dimensional function spaces $U,V$, the goal of \emph{linear operator learning}~\cite{boulle2023,de2023convergence,boulle2024operator,boulle2024mathematical,boulle2023learning,boulle2022learning,tabaghi2019learning} is to estimate a linear map $R:U\to V$ given a dataset $\{S_1 u_j,S_2 Ru_j\}_{j=1}^J$, where each $u_j$ is an element of $U$, and $S_1,S_2$ model sampling/discretization (when necessary).

The structure of the dataset and operator motivate learning algorithms and theoretical sampling guarantees.
For instance, Boull{\'e} et al. show that $R$ can be recovered up to error $\epsilon>0$ from $J=\mathrm{polylog}(1/\epsilon)$ if it is the Green's function of an elliptic PDE~\cite{boulle2023}, provided that the  $u_j$ are hierarchically structured in a way adapted to the regularity of the matrix representation of these operators.
Similarly, if $U,V$ are separable Hilbert spaces and $R$ is diagonalized by known orthonormal bases for $U,V$, then sampling rates when learning from instances of random noise can be derived~\cite{de2023convergence}.
Related results on PAC-learnability follow from assumptions on the singular values of the operator~\cite{tabaghi2019learning}.

Beyond linear operator learning, nonlinear problems have often been approached using deep learning methods~\cite{lecun2015deep,chen2023deep,lanthaler2022error,gin2021deepgreen,lu2021learning,bubba2020deep,de2019deep,adcock2020deep,fan2020solving,Raissi2019PINN,Sirignano2018DGM}, perhaps most famously using Fourier Neural Operators (FNOs)~\cite{li2020fno,kovachki2021universal,pathak2022fourcastnet,li2023fourier}.
FNOs learn function-to-function maps parameterized by interleaved truncated Fourier transforms and pointwise nonlinearities, yielding a translation-equivariant nonlinear operator.

Many approaches to operator learning are undergirded, either implicitly or explicitly, by a known geometric structure associated with the problem.
In the case of learning solution operators for elliptic PDEs, the high-frequency features of the pointwise evaluation functional $[Ru](x)$ (asymptotically) depend only on the values of $u$ in a small neighborhood of $x$.
Although the abstract framework put forth in~\cite{de2023convergence} does not presume a geometric domain over which the function spaces are defined, it still hypothesizes that a simple relationship is known between representations in orthonormal bases, which can be interpreted as an abstract form of the geometric assumption on which the other methods rest.

\section{Double Fibration Transforms}

We consider operators that model integrals of geometric relations between domains.
Let $Y,X$ be compact smooth manifolds referred to as the \emph{measurement domain} and the \emph{target domain}, respectively.
The spaces of distributions\footnote{For the reader unfamiliar with the theory of distributions, think of them as scalar-valued functions, but with allowance for singularities such as Dirac delta functions.} on $Y,X$ are written $\mathcal{D}'(Y),\mathcal{D}'(X)$.
Let the \textit{incidence relation} $Z\subset Y\times X$ be compact and equipped with a smooth nonvanishing measure $\mu$, and denote the respective projection maps by $p:Z\to Y$ and $q:Z\to X$.
The relationship of $Z$ to $Y,X$ via these maps is indicated by the diagram $\dft{Y}{p}{Z}{q}{X}$.

This defines an operator $R:\mathcal{D}'(X)\to\mathcal{D}'(Y)$ that acts as follows.
For each $y\in Y$, put $G_y=(q\circ p^{-1})(y)\subset X$ for each $y\in Y$, and define:
\begin{equation}\label{eq:fiber-integral}
    Ru(y) = \int u(x)d\mu_y(x),
\end{equation}
where $\mu_y$ is a measure on $X$ supported on $G_y$, determined by the measure $\mu$ on $Z$.
Succinctly, $Ru$ computes the pushforward by $p$ of the pullback of $u$ by $q$:
\begin{equation*}
    Ru(y) = p_{*}(u\circ q)(y).
\end{equation*}
\begin{example}[Diffeomorphism]\label{example:diffeomorphism}
    Let $X$ and $Y$ be smooth manifolds with a diffeomorphism $f:Y\to X$, and define the operator $R:\mathcal{D}'(X)\to\mathcal{D}'(Y)$ as the \emph{pullback of the diffeomorphism} $f$, that is, $Ru(y)=(u\circ f)(y)$.
    Then, $G_y$ is the singleton set $\{f(y)\}$ for all $y\in Y$.
    A trivial instance is when $Z$ is given by the diagonal $\Delta(X)\subset X\times X$.
    In this case, $R$ is simply a multiplier.
\end{example}

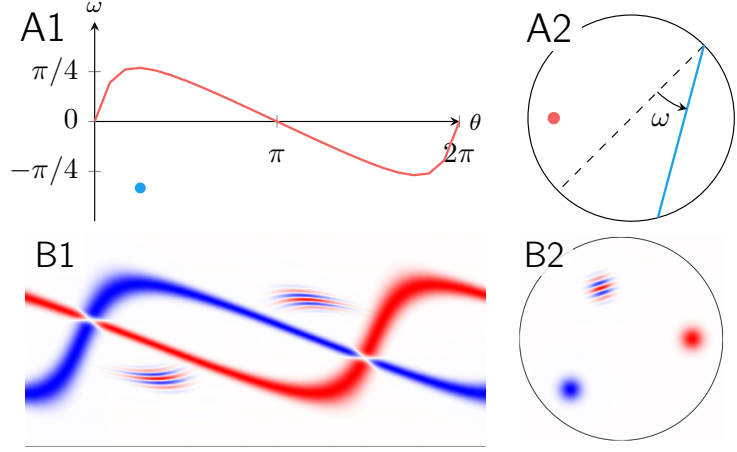
\begin{figure}
\centering
\resizebox{0.6\linewidth}{!}{\begin{tikzpicture}

    \begin{scope}
    \draw (0, 0) circle [radius=1.25cm];

    \draw[dashed] (45:1.25) -- (225:1.25);
    \draw[pastelBlue,thick] (45:1.25) -- (285:1.25);
    \draw[draw=none] (285:1.25) coordinate (A) -- (45:1.25) coordinate (B) -- (225:1.25) coordinate (C)
    pic ["$\omega$",angle eccentricity=1.3,draw,-stealth,angle radius=0.8cm] {angle = C--B--A};

    \fill [pastelRed] (180:0.9375) circle (0.075);

    \node [fill=white,anchor=north west] at (-1.4,1.4) {\Large\textsf{A2}};
    \end{scope}

    \begin{scope}[xshift=-6.5cm, yshift=-1.25cm]
    \begin{axis}[
    xmin=0, xmax=360,
    ymin=-90, ymax=90,
    width=6cm, height=4cm,
    axis x line=middle,
    axis y line=left,
    xtick={180,360},
    xticklabels={$\pi$,$2\pi$},
    ytick={-45,0,45},
    yticklabels={$-\pi/4$,$0$,$\pi/4$},
    typeset ticklabels with strut,
    clip=false
    ]
    
    \addplot[pastelRed, domain=0:360, thick] {asin(3*sin(x)/(4*sqrt(1.5626 - 3*cos(x)/2)))};
    \node at (axis cs:45,-60) [circle, scale=0.4, fill=pastelBlue] {};

    \node at (axis cs:360,0) [anchor=west] {\footnotesize$\theta$};
    \node at (axis cs:0,90) [anchor=south] {\footnotesize$\omega$};

    \node [fill=white,anchor=north east] at (rel axis cs: -0.05,1.1) {\Large\textsf{A1}};
    
    \end{axis}
    \end{scope}
\end{tikzpicture}}
\resizebox{0.6\linewidth}{!}{\def\graphicdir{figs/data/fig_packets}

\begin{tikzpicture}
    \begin{groupplot}[
    group style={
    group size=2 by 1,
    group name=myGroup,
    horizontal sep=0.4cm,
    },
    tick style={draw=none},
    xticklabel=\empty,
    yticklabel=\empty,
    width=4.5cm,
    height=4.5cm,
    xmin=-1, xmax=1,
    ymin=-1, ymax=1,
    clip mode=individual
    ]


    \nextgroupplot[width=7.875cm,xmin=-1.9,xmax=1.9,ymin=-0.9,ymax=0.9]
    \addplot graphics[xmin=-2,xmax=2,ymin=-1,ymax=1] 
    {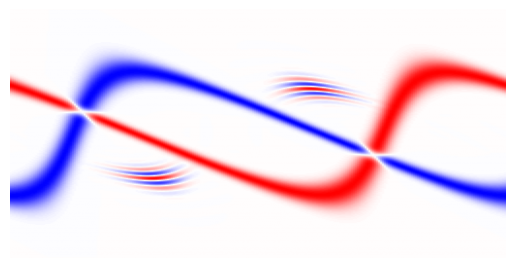};
    \node [fill=white, fill opacity=0.8, text opacity=1, anchor=north west] at (rel axis cs: 0,1) {\Large\textsf{B1}};

    \nextgroupplot[axis line style={draw=none}]
    \addplot graphics[xmin=-1,xmax=1,ymin=-1,ymax=1] 
    {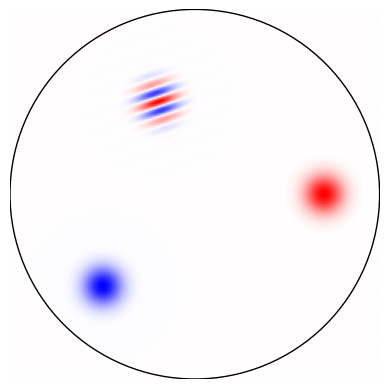};
    \node [fill=white, fill opacity=0.8, text opacity=1, anchor=north west] at (rel axis cs: 0,1) {\Large\textsf{B2}};

    \end{groupplot}
\end{tikzpicture}}
\caption{
\textbf{(A)} Geometry of the Radon transform~(\cref{example:radon}).
\textit{(A1)} Measurement domain $Y$, with a point $y$ (blue) and a fiber $H_x$ (red).
\textit{(A2)} Target domain $X$, with the fiber $G_y$ corresponding to $y$ (blue) and the point $x$ defining the fiber $H_x$ (red).
\textbf{(B)} The Radon transform of Gabor atoms.
\textit{(B1)} Radon transform of
a function $u(x)$ given by the sum of three Gabor atoms, two of which are unmodulated \textit{(B2)}.
That is, $\hat\xi=0$ for an unmodulated atom $g_{\hat x,\hat\xi}$.
Each unmodulated Gabor atom $g_{\hat x,0}$ is mapped to a neighborhood of the whole fiber $H_{\hat x}$, while the modulated Gabor atom is only mapped to a small region dictated by the conormal bundle of $Z$.
}
\label{fig:ray-fibration}
\end{figure}

\begin{example}[Radon Transform]\label{example:radon}
    Let $X$ be the open unit disc in $\mathbb{R}^2$ and $Y=\partial X\times (-\pi/2,\pi/2)$.
    For $(\theta,\omega)\in Y$, let $G_y$ consist of all points $x$ such that
    $x$ lies on the line segment with an endpoint at $\theta$ and angle $\omega$ relative to the inward-pointing normal vector.
    The \emph{Radon transform} is the operator $R:\mathcal{D}'(X)\to\mathcal{D}'(Y)$ such that
    \begin{equation*}
        Ru(y) = \int_{G_y}u(x)dx.
    \end{equation*}
    The geometric relationship between $Y$ and $X$ for the Radon transform is illustrated in \cref{fig:ray-fibration}~(A).
\end{example}

Denote the dimensions of $Y,X$ by $N=\dim(Y)$ and $n=\dim(X)$, respectively, and assume $N\geq n$.
Suppose that the incidence relation $Z\subset Y\times X$ is a smooth submanifold of codimension $n''\leq n$, with the additional assumption that the projection maps $p,q$ are submersions.
Then, we call the resulting operator a \emph{double fibration transform}, 
since this condition implies that $p,q$ are fibration~\cite{GuilleminSternberg,Guillemin1985,mazzucchelli2023}.
The set $G_y$ is called the \emph{fiber} of $y$.
It is also useful to define the fiber of $x\in X$ as $H_x=(p\circ q^{-1})(x)$.
The codimension of the incidence manifold $Z$ dictates the codimension of the fibers $G_y$ and $H_x$, and is referred to as the codimension of the transform.
Operators of this type are considered in integral geometry~\cite{GuilleminSternberg,Guillemin1985}, with applications in a variety of imaging modalities, including radar~\cite{StefanovUhlmann_SAR}, seismic~\cite{FGGN,deHoopStolk}, and microscopy~\cite{quinto2008local}.
When $R$ describes an operator that computes integrals over submanifolds of some fixed codimension $n''$, while not necessarily being a double fibration transform, we refer to it generically as a \emph{geometric integral operator}.

\begin{example}[Euclidean Ray Transform]\label{example:ray}
    Let $X$ be the open unit disc in $\mathbb{R}^n$ for some $n\geq 2$ and $Y=\partial_{in}SX$ be the inward-pointing sphere bundle, that is, the set of all inward-pointing unit rays with basepoint on the boundary of $X$.
    For each $y\in Y$, let $G_y$ consist of all points $x\in X$ that lie along the extension of the ray $y$.
    The Euclidean Ray transform is defined via integration over the fibers as in \cref{example:radon}: $Ru(y)=\int_{G_y}u(x)dx$.
\end{example}
\begin{example}[Spherical Mean Transform]\label{example:circles}
    Let $X=\mathbb{R}^n$ and $Y\subset\mathbb{R}^n\times\mathbb{R}^{>0}$.
    Define the incidence submanifold $Z\subset Y\times X$ as
    \begin{equation*}
        Z = \{(y,r;x):\|y-x\|^2=r^2\}.
    \end{equation*}
    Coupled with a smooth measure $\mu$, this defines a double fibration transform such that for each $(y,r)\in Y$, the fiber $G_y$ is the sphere centered at $y$ with radius $r$.
\end{example}
We remark that the Euclidean ray transform is a simple example of the more general \emph{geodesic ray transform}~\cite{HolmanUhlmann,monard2014} in a Riemannian metric, which we consider later in \eqref{eq:geodesic-ray-transform}.
Similarly, the spherical mean transform is a simple case of a generalized Radon transform~\cite{de2009seismic}, which can also be defined in non-Euclidean spaces.

\section{Data-Efficient Learning}

Double fibration transforms map $L^2(X)$ to $L^2(Y)$.
We now consider if (approximations) of these operators are \emph{learnable}.
In particular, we ask how many examples $\{u_j,Ru_j\}_{j=1}^J$ are necessary to form a close estimate of $R$.

\subsection{Phase Space Transformations}

The submanifold $Z\subset Y\times X$ is a relation between points in $Y$ and points in $X$.
The geometry of $Z$ dictates how certain properties of functions $u(x)$ are reflected in $Ru(y)$.
We probe these relationship using time-frequency atoms, defined as follows.

Assume that $Y$ is a Euclidean domain, that is, a compact subset of $\mathbb{R}^{N}$.
Then, for a window function $h\in C_0^{\infty}(\mathbb{R}^{N})$ that can be written as the tensor product of window functions on $\mathbb{R}$, we define the Gabor atoms
\begin{equation*}
    h_{\hat y,\hat \eta}(y) =
    e^{i2\pi\langle \hat\eta, y-\hat y\rangle}h(y-\hat y)
\end{equation*}
indexed by $(\hat y,\hat\eta)\in T^*Y$.
A similar construction is used for a separable window function $g\in C_0^{\infty}(\mathbb{R}^n)$ to define Gabor atoms $g_{\hat x,\hat\xi}(x)$ on $X$ indexed by $(\hat x,\hat\xi)\in T^*Y$.
We assume that the windows $h,g$ are chosen such that each has ``small'' support.
From the geometry of the double fibration transform, then, it is clear that $\langle h_{\hat y,\hat\eta}, Rg_{\hat x,\hat\xi}\rangle\neq 0$ only if $\hat x$ is contained in a sufficiently small neighborhood of $G_{\hat y}$.
This is pictured for the Radon transform in \cref{fig:ray-fibration}~(B).

The relationship between the support of $Ru$ and that of $u$ is obvious given the nature of a double fibration transform as the pushforward of a pullback following the diagram 
$\dft{Y}{p}{Z}{q}{X}$.
The geometry of the incidence submanifold also dictates how singularities in $u$ are propagated by $R$: that is, how the wavefront set~\cite{Hormanderacta} $\wf(Ru)$ relates to $\wf(u)$.
At a high level, the wavefront set $\wf(u)$ of a distribution $u\in\mathcal{D}'(X)$ consists of points $(x,\xi)\in T^*X$ such that in an arbitrarily small neighborhood of $x$, $u$ appears ``non-smooth'' in the direction $\xi$, as signaled by slow decay of the Fourier transform.
A typical example is a function on $\mathbb{R}^2$ that is smooth away from a $C^2$ curve: the wavefront set consists of $(x,\xi)$ where $x$ lies on the curve and $\xi$ is normal to the curve at $x$.
A similar definition holds for defining the wavefront set for distributions on $Y$.
We refer the reader to \cite{Hormanderacta} for a more precise definition of the wavefront set.

Since $\wf(Ru)\subset T^{*}Y$ and $\wf(u)\subset T^{*}X$, any relation between the two is a subset of $T^{*}Y\times T^{*}X$.
Indeed, this relation is given by the conormal bundle\footnote{The conormal bundle of $Z$ consists of points $(y,x)\in Z$ paired with covectors $(\eta,\xi)$ perpendicular to $Z$ at $(y,x)$.} of $Z$.
\begin{proposition}[{\cite[Theorem~2.2]{mazzucchelli2023}}]\label{prop:of-singularities}
    Let $R$ be a double fibration transform for an incidence submanifold $Z$ with smooth nonvanishing measure $\mu$.
    Then, for any $u\in\mathcal{D}'(X)$ and $(y,\eta)\in T^{*}Y$,
    it holds that $(y,\eta)\in\wf(Ru)$ if and only if there is an $(x,\xi)\in\wf(u)$ such that
    \begin{equation}\label{eq:conormal-transpose}
        (y,-\eta;x,\xi)\in N^{*}Z,
    \end{equation}
    where $N^{*}Z$ denotes the conormal bundle of $Z\subset Y\times X$.
\end{proposition}
Denoting the set of all $(y,\eta;x,\xi)$ that satisfy \eqref{eq:conormal-transpose} by $N^*Z'$, \cref{prop:of-singularities} indicates that $\wf(u)\subseteq(Q\circ P^{-1})(\wf(Ru))$, where $P:N^*Z'\to T^*Y$ and $Q:N^*Z'\to T^*X$ are the projection maps.
That is, we ``lift'' the diagram describing a double fibration transform to the cotangent bundle:
$\dft{T^*Y}{P}{N^*Z'}{Q}{T^*X}$.
When $Q\circ P^{-1}$ is injective, the upper bound of \cref{prop:of-singularities} is exact.
\begin{definition}[{Bolker Condition~\cite{GuilleminSternberg,Guillemin1985,mazzucchelli2023}}]\label{defn:bolker}
    A double fibration transform with incidence submanifold $Z$ is said to satisfy the \emph{Bolker condition} if $P:N^*Z'\to T^*Y$ is an injective immersion.
\end{definition}
An interpretation of the Bolker condition is that a reconstruction of $u$ from measurements $Ru$ will not create any new singularities~\cite{kunstmann2023seismic}.
In particular, for a double fibration transform $R$ satisfying the Bolker condition, the normal operator $R^*R$ is a pseudodifferential operator, and thus has the relationship $\wf(R^*Ru)\subseteq\wf(u)$.

The cotangent bundles $T^*Y$ and $T^*X$ are often referred to as \emph{phase space}.
The conormal bundle of the incidence submanifold is thus a relation in phase space, signified by the diagram $\dft{T^*Y}{P}{N^*Z'}{Q}{T^*X}$.
Under the Bolker condition, this relation is be described by a map $\chi=Q\circ P^{-1}$, which we call the \emph{Bolker map}.
The Bolker map is positive homogeneous in the phase variable, \textit{i.e.}, if $\chi(y,\eta)=(x,\xi)$, then for any $s>0$, $\chi(y,s\eta)=(x,s\xi)$.

Although the Bolker map describes a relationship between the \emph{cotangent bundles}, it can be used to understand how $R$ acts on \emph{Gabor atoms} based on their localization in phase space.
For the window function $h$, it is of course impossible for its Fourier transform to be compactly supported, by the Paley-Weiner theorem.
However, due to the smoothness of $h$, its Fourier transform is concentrated in a neighborhood of the origin: thus, we say that $h$ has effective support in some $U_{0,0}\subset T^*\mathbb{R}^N$, where $U$ is a bounded neighborhood of $(0,0)$.
It follows, then, that $h_{\hat y,\hat\eta}$ has effective support in $U_{\hat y,\hat\eta}$, which is simply the translation of $U_{0,0}$ to be centered at $(\hat y,\hat\eta)$.
We denote the effective phase space support of Gabor atoms $g_{\hat x,\hat\xi}$ by $V_{\hat x,\hat\xi}\in T^*\mathbb{R}^n$, following similar logic.

Heuristically, we expect $|\langle h_{\hat y,\hat\eta}, Rg_{\hat x,\hat\xi}\rangle|$ to be large when $V_{\hat x,\hat\xi}$ overlaps with $\chi(U_{\hat y,\hat\eta})$, and to be small otherwise.
That is, the operator $R$ is \emph{well-organized} when represented using suitable families of Gabor atoms, in the sense that it concentrates near pairs $(\hat y,\hat\eta;\hat x,\hat\xi)\in N^*Z'$, in accordance with \cref{prop:of-singularities}.
\begin{remark}
    We establish the claim that $R$ is well-organized when represented using Gabor atoms more precisely in the SI.
    To do so, we leverage the fact that $R$ is a Fourier Integral Operator (FIO) with a symbol independent of the frequency variable, and is thus in the class $(\rho,\delta)$ for any $\rho\geq 0$ and $\delta\geq 0$, making it amenable to methods of time-frequency analysis~\cite{Hormanderacta}.
\end{remark}
Thinking of this representation of $R$ as being indexed by $T^*Y \times T^*X$, this means that each $(y,\eta)\in T^*Y$ has very few corresponding large elements in $T^*X$.
This concept is illustrated in \cref{fig:ray-fibration}~(B), where the Radon transform is shown to map Gabor atoms with nonzero modulation term $\hat\xi$ to objects with similar time-frequency localization.

\subsection{Compressibility}

We now describe the number of input-output examples needed to learn a double fibration transform on $L^2(X)$.
The following theorem establishes a data-efficiency result, leveraging the fact that $R$ is well-organized when represented using Gabor atoms, thus making it amenable to methods of compressive sensing~\cite{candes2006stable,baraniuk2007compressive}:
\begin{theorem}\label{theorem:data-efficiency}
    Let $R:\mathcal{D}'(X)\to\mathcal{D}'(Y)$ be a double fibration transform satisfying the Bolker condition, and let $\epsilon>0$ be sufficiently small.
    Then, for all $r>1$, there exists a randomized algorithm yielding an estimate $\widehat{R}$ of $R$ from $J\approx ({C_r}/{\epsilon})^{1/r}\log(1/\epsilon)$ training samples $\{u_j,Ru_j\}_{j=1}^J$ such that
    \begin{equation*}
        \sup_{(\hat y,\hat\eta)\in T^*Y}\big|\langle h_{\hat y,\hat\eta}, (\widehat{R}-R)u\rangle\big|<\epsilon\|u\|_{L^2}
    \end{equation*}
    for all $u\in L^2(X)$ with high probability, where $C_r>0$ depends only on $r$.
\end{theorem}
The construction is such that the training inputs $\{u_j\}_{j=1}^{J}$ are bandlimited, but otherwise unstructured--that is, our sampling rates are derived using training data that is experimentally realistic, rather than being algorithmically adapted to the operator~\cite{boulle2023learning}.
Moreover, the bound is in terms of the norm of $\widehat{R}-R$ as an operator, rather than its performance on a test distribution, making this result more general than PAC guarantees~\cite{tabaghi2019learning}.
Additionally, by \cite[Theorem~2.2]{mazzucchelli2023}, the double fibration transform $R$ is a smoothing operator of order $s:=(2n''-(n+N))/4<0$~\cite[Theorem~4.3.2]{Hormanderacta}.
Because of this, the $\epsilon$-approximation of \cref{theorem:data-efficiency} indicates the effective bandwidth of the estimate $\widehat{R}$, given by $B\approx\epsilon^{1/s}$.
\begin{takeaway}\label{takeaway:first}
    Learning smooth integral geometry is not cursed.
\end{takeaway}
\Cref{theorem:data-efficiency} resembles the data-efficiency result of \cite[Theorem~1]{boulle2023}, where a self-adjoint Hilbert-Schmidt operator $H$ arising from the Green's function of an elliptic PDE is approximated as a multiscale hierarchical matrix organized around a known diagonal geometry.
Sample bounds are derived to learn this matrix to spatial resolution $\epsilon>0$ from $J=\polylog(\epsilon^{-1})$ samples.\footnote{We refer the reader to Eq.~7 of the SI in \cite{boulle2023} for an explanation of this.}
The training examples are carefully constructed to align with the hierarchical low-rank structure of the operator.

In our setting, the operator has \emph{unknown, smooth geometry}, which results in \emph{superalgebraic} rates.
This is due to the fact that the constants $C_r$ are dictated by higher-order derivatives of the representation of $Z$ locally as a graph $y''=b(x,y')$.
For general smooth geometry, the high-order derivatives of $b$ can grow arbitrarily quickly with $r$--hence, the representation of $R$ using Gabor atoms exhibits superalgebraic decay, but not necessarily subexponential decay.

Indeed, if the representation of $R$ using Gabor atoms exhibits \emph{subexponential} concentration properties that are, the double fibration transform would exhibit (Gevrey-)analytic structure~\cite{cordero2015exponentially}.
Double fibration transforms enjoying such structure arise naturally in theoretical physics in the context of AdS/CFT correspondence ~\cite{jokela2025}. 
These conditions would yield $C_r\lesssim (r!)^\varsigma$ for some $\varsigma\geq 1$, so that \cref{theorem:data-efficiency} implies a $\polylog(\epsilon^{-1})$ sampling rate by taking $r=-\log(\epsilon)$, even in our setting where we learn from random bandlimited functions.

\section{Smoothed Levelset Methods}

Recall that our motivation to study learning rates for double fibration transforms is that the existing rates in the literature address either operators that do not encode any geometry, or where the geometry is known \textit{a priori}.
Studying double fibration transforms is a natural stepping stone towards theory for more complex operators that propagate singularities or change the domain of the function (\textit{e.g.}, functions on $X$ are transformed to functions on $Y$).
The ultimate goal is to have a theory for state-of-the-art deep learning approaches to the operator learning problem for very general classes of maps.

As a first step in this direction, in the remainder of this paper we design a deep learning architecture to effectively parameterize double fibration transforms.
The \textit{desiderata} of our design include
\begin{enumerate}[label=(\roman*)]
    \item Interpretability, enabling \textit{post hoc} analysis to extract the incidence relation and measure\label{item:interp}
    \item Differentiability, for gradient-based learning of inherently singular objects\label{item:smooth}
    \item Universality, \textit{i.e.}, the capacity to approximate any double fibration transform\label{item:universal}
    \item Stability under discretization, for practical use on digital computers and under limited measurement regimes.\label{item:discrete}
\end{enumerate}

\subsection{Levelset Integral Kernels}

Let $R:\mathcal{D}'(X)\to\mathcal{D}'(Y)$ be a double fibration transform of codimension $n''$.
\begin{assumption}\label{assump:levelset}
  There is a $C^2$ function $f_0:Y\times X\to\mathbb{R}^{n''}$ such that $Z=f_0^{-1}(0)$, where $0$ is a regular point of $f_0$.
  Moreover, $|\nabla_x f_0(y,x)|=1$ for all $(y,x)\in Z$.
\end{assumption}
This is easy to satisfy if $Z$ is an orientable submanifold of $Y\times X$, for instance.
Motivated by \textit{desideratum}~\ref{item:smooth}, we construct a smooth approximation to $R$, ameliorating the difficulties in training a singular object on nonsmooth data.
Assume that $X$ is a compact subset of $\mathbb{R}^{n}$.
Consider a smoothed function $u_\lambda$ computed via convolving some $u\in C^1(X)$ with a gaussian kernel with effective width proportional to $1/\sqrt{\lambda}$, for some $\lambda>0$.
Denoting the Lebesgue measure on $X$ by $\nu$, a first-order approximation under \cref{assump:levelset} yields
\begin{align*}
    Ru_\lambda(y) 
    &=
    \left(\frac{\lambda}{\pi}\right)^{n/2}
    \iint
    e^{-\lambda|x-\tilde{x}|^2}u(x) 
    d\nu(x)d\mu_y(\tilde{x}) \\
    &\approx 
    \left(\frac{\lambda}{\pi}\right)^{n''/2}
    \int_{X}e^{-\lambda|f_0(y,x)|^2}a(y,x)u(x)d\nu(x),
\end{align*}
where $a(y,x)$ is the \emph{amplitude} of the operator, proportional to the measure $\mu_y$ extended to a $O(1/\sqrt{\lambda})$-neighborhood of $G_y$ along the normal bundle.
For $(y,x)\in f^{-1}(0)$, the integral of $\exp(-|f_0(y,x)|^2)a(y,x)$ normal to the fiber $G_y$ at $x$ is proportional to $a(y,x)$.
At the same time, the integral of $\exp(-|f_0(y,x)/a(y,x)|^2)$ normal to the fiber is also on the order of $a(y,x)$, effectively trading the ``height'' of the former function for ``width'' proportional to $a(y,x)$.
Making the further assumption that $C\geq a(y,x)$ for $(y,x)\in Z$, and since $a(y,x)$ is strictly positive, the exponential normal to the fiber $G_y$ will always concentrate in an interval of width $O(\sqrt{C/\lambda})$.
We put $f(y,x)=f_0(y,x)/a(y,x)$, yielding
\begin{equation*}
    Ru_{\lambda}(y) \approx
    \left(\frac{\lambda}{\pi}\right)^{n''/2}
    \int_{X}e^{-\lambda|f(y,x)|^2}u(x)d\nu(x).
\end{equation*}
Finally, observe that $\int_X\exp(-\lambda|f(y,x)|^2)d\nu(x)\asymp \lambda^{-n''/2}$, so we replace the scaling factor by this integral combined with a weight function $w(y)>0$:
\begin{equation*}
    Ru_{\lambda}(y) \approx
    w(y)
    \frac{\int_{X}
    e^{-\lambda|f(y,x)|^2}%
    u(x)d\nu(x)}
    {\int_{X}e^{-\lambda|f(y,x)|^2}d\nu(x)}.
\end{equation*}
Note that $f$ encodes both the incidence relation $Z=f^{-1}(0)$ and the amplitude $a(y,x)$.
Based on a double fibration transform of codimension $n''\leq n$, $f$ is specified as a map with codomain $\mathbb{R}^{n''}$.
By zero-padding, a more general map $f:Y\times X\to\mathbb{R}^{m}$ with $m\geq n''$ yields an identical operator.
We summarize this construction as follows:
\begin{definition}\label{defn:levelset-integral-kernel}
    For domains $Y,X$, let $f:Y\times X\to\mathbb{R}^m$ and $w:Y\to\mathbb{R}^{>0}$ be smooth, and let $\lambda>0$.
    Then, the integral operator acting on distributions $u\in\mathcal{D}'(X)$ as
    \begin{equation}\label{eq:levelset-kernel}
        L^{\lambda}u(y) =
        w(y)
        \frac{\int_{X}
        e^{-\lambda|f(y,x)|^2}%
        u(x)d\nu(x)}
        {\int_{X}e^{-\lambda|f(y,x)|^2}d\nu(x)}
    \end{equation}
    is called the \emph{levelset integral kernel} of $f,w$ at scale $\lambda$.
\end{definition}
\begin{remark}
    The representation of a geometric object implicitly is connected to levelset methods~\cite{osher2001level} used in computer graphics and PDEs, and is particularly in line with using neural networks to implicitly represent surfaces as levelsets of suitable functions~\cite{michalkiewicz2019deep,gropp2020implicit,sitzmann2020implicit,mildenhall2021nerf,saragadam2023,roddenberry2024}.
    Implicit representations have also been suggested for parameterizing diffeomorphisms in the forward model for Cryo-electron tomography~\cite{debarnot2024ice}, or for other applications where deformations of a domain are required for image synthesis~\cite{park2021nerfies}.
\end{remark}
Observe that this approximation to $Ru_\lambda$ is no longer computed using the smoothed function $u_\lambda$, but rather nonuniformly smoothes the operator according to the chosen $\lambda>0$.
Moreover, the replacement of the $\lambda^{n''/2}$ scaling by the denominator of the integrand obviates the need for \emph{a priori} knowledge of the codimension $n''$.

We now study the behavior of $L^{\lambda}$ for general $f$, as opposed to one derived from a particular double fibration transform.
We work under the following two assumptions.
\begin{assumption}\label{assump:nontrivial-levelset}
    For all $y\in Y$, there is some $x\in X$ such that $f(y,x)=0$.
\end{assumption}
\begin{assumption}\label{assump:levelset-gradient}
    For $(y_0,x_0)\in f^{-1}(0)$, the Jacobian of $f$ with respect to the $x$ variable in a small neighborhood of $x_0$ has constant rank $n''$ .
\end{assumption}
\Cref{assump:levelset-gradient} is a generic property of smooth $f$ when $m=n''$, by Sard's lemma.
For $f$ satisfying \cref{assump:nontrivial-levelset,assump:levelset-gradient}, we define the incidence submanifold $Z_f=f^{-1}(0)$, noting that it is indeed a submanifold of codimension $n''$.
The projection maps are denoted $p:Z_f\to Y$ and $q:Z_f\to X$ as before.
We then define a smooth measure $\mu_f$ on $Z_f$ so that for $(y,x)\in Z$,
\begin{equation}\label{eq:implicit-measure}
    \begin{aligned}
        \mu_f(y,x) &= \frac{w(y)a_f(y,x)}{\int_{G_y}a_f(y,x')d\nu_y(x')}\nu_y(x), \\
        a_f(y,x) &= \gdet_{n''}^{-1/2}\left((\nabla_x f^\top \nabla_x f)(y,x)\right),
    \end{aligned}
\end{equation}
where $\gdet_{n''}$ denotes the product top $n''$ eigenvalues of a matrix, $\nabla_x$ denotes the Jacobian with respect to the $x$ variable, and $\nu_y$ denotes the uniform measure on the fiber $(q\circ p^{-1})(y)$ with respect to the Lebesgue measure $\nu$.
\begin{lemma}[Gaussian Approximation]\label{lemma:gauss-approximation}
    For $f$ satisfying \cref{assump:nontrivial-levelset,assump:levelset-gradient}, smooth $w:Y\to\mathbb{R}^{>0}$, and $u\in C^1(X)$,
    \begin{enumerate}
        \item The submanifold $Z_f=f^{-1}(0)$ with measure $\mu_f$ (defined in \eqref{eq:implicit-measure}) determines a geometric integral operator $L:\mathcal{D}'(X)\to\mathcal{D}'(Y)$ with codimension $n''$.\label{item:lemma-interp} \label{lemma:gauss:one}
        \item For all $y\in Y$ and $\lambda$ sufficiently large, the levelset integral kernel satisfies
        \begin{equation*}
            |Lu(y) - L^{\lambda}u(y)|
            \leq
            C\lambda^{-n''/2}\|u\|_{C^1(X)},
        \end{equation*}
        for some $C>0$ depending on $f$.\label{item:lemma-smooth} \label{lemma:gauss:two}
    \end{enumerate}
\end{lemma}
\Cref{lemma:gauss-approximation} indicates that the levelset integral kernel fulfills \textit{desideratum}~\ref{item:interp}, where the incidence relation and measure $Z_f$ and $\mu_f$ can be calculated directly from the parameterization.
Despite the integral kernel becoming singular as $\lambda\to\infty$, the smoothed version using $\lambda<\infty$ fulfills \textit{desideratum}~\ref{item:smooth}, allowing for smooth optimization techniques to be used when learning $f$ from data.
\begin{definition}\label{defn:singular-levelset-integral-kernel}
    For a function $f$ satisfying \cref{assump:nontrivial-levelset,assump:levelset-gradient} and smooth map $w:Y\to\mathbb{R}^{>0}$, the geometric integral operator $L$ parameterized by $f,w$ according to \cref{lemma:gauss-approximation} is called the \emph{singular levelset integral kernel} of $f,w$.
\end{definition}
\begin{example}[Spherical Mean III]\label{example:circles-factorized-levelset}
    Consider a target domain $X\subset\mathbb{R}^2$ and the measurement domain $Y=\mathbb{R}^2\times\mathbb{R}^{>0}$.
    For $(y,r)\in Y$ and $x\in X$, define $f(y,r;x)$ as
    \begin{equation*}
        f(y,r;x) = |x|^2 + |y|^2 - 2\langle x,y\rangle - r^2.
    \end{equation*}
    The function $f$ is smooth on $Y\times X$ and defines a double fibration transform with incidence relation $Z=f^{-1}(0)$, which computes the spherical mean transform as defined in \cref{example:circles}, with measure of integration uniform over each fiber.
\end{example}

\section{Efficient Factorization}

Writing a double fibration transform in terms of a smooth map $f(y,x)$ suggests that $f$ can be parameterized with methods amenable to optimization, \textit{e.g.}, by a neural network.
We now develop a factorization of $f$ as a product of functions evaluated on $Y$ and $X$ separately.
To see why, assume that each evaluation of $f$ is relatively expensive, such as a multilayer perceptron requiring multiple matrix multiplications and elementwise activation function calls.
If we approximate integration over $X$ by a sum over $k_X$ points, and wish to evaluate $R^{\lambda}u(y)$ over a set of $k_Y$ points in $Y$, this requires $k_Y\cdot k_X$ evaluations of a neural network.

On the other hand, suppose that $f(y,x)=\langle \psi_Y(y), \psi_X(x)\rangle+b(y)$ for suitable functions $\psi_Y,\psi_X,b$ (to be described in more detail shortly).
Each evaluation of each function $\psi_Y,\psi_X,b$ is also expensive, but the computation of $f(y,x)$ from these quantities is relatively inexpensive, consisting of a simple inner product and sum.
Then, to evaluate $R^{\lambda}u(y)$ over the discretizations of $Y$ and $X$ requires $2k_Y + k_X$ neural network evaluations, yielding significant savings when $k_Y$ and $k_X$ are large.
Moreover, for a fixed discretization of the domain $X$ by $k_X$ points, the computation of $\{\psi_X(x_i)\}_{i=1}^{k_X}$ only needs to be carried out once, after which the embeddings can be stored in memory.
After this, for any new point $y\in Y$ and function $u:X\to\mathbb{R}$, the evaluation of $L^\lambda u(y)$ only requires the evaluation of $\psi_Y(y)$ and $b(y)$, along with the subsequent integral/sum in \eqref{eq:levelset-kernel}.

More specifically, let $\psi_Y:Y\to\mathbb{R}^{d\times m},\psi_X:X\to\mathbb{R}^{d\times m},b:Y\to\mathbb{R}^{m}$ be learnable smooth maps.
We abuse notation so that $\langle\psi_Y(y),\psi_X(x)\rangle+b(y)\in\mathbb{R}^{m}$, where the inner product $\langle\cdot,\cdot\rangle$ is taken over the ``$d$-axis.''
Define
\begin{equation}\label{eq:factorized-levelset-kernel}
    f(y,x) := \langle\psi_Y(y),\psi_X(x)\rangle+b(y).
\end{equation}
It is not strictly necessary, but we put $b(y)=-\int_{X}\langle\psi_Y(y),\psi_X(x)\rangle d\nu(x)$, which enforces \cref{assump:nontrivial-levelset} when $m=1$.
\begin{example}[Spherical Mean II]\label{example:circles-levelset}
    Consider the spherical mean transform parameterized following \cref{example:circles-levelset}.
    The function $f(y,r;x)$ in that example can be written as
    \begin{equation*}
        \begin{aligned}
            \psi_Y(y,r) &= (1,-2y_1,-2y_2,y_1^2+y_2^2-r^2) \\
            \psi_X(x) &= (x_1^2+x_2^2,x_1,x_2,1) \\
            f(y,r;x) &= \langle\psi_Y(y,r), \psi_X(x)\rangle.
        \end{aligned}
    \end{equation*}
    yielding a factorized levelset integral kernel that computes the spherical mean transform defined in \cref{example:circles}.
\end{example}

\begin{remark}\label{remark:attention}
In factorized form, the levelset integral kernel is reminiscent of cross-attention~\cite{bahdanau2014neural,Vaswani2017AttentionIA}.
Indeed, attention-based methods have been used for operator learning problems~\cite{zappala2024}.
The factorized levelset integral kernel $L^{\lambda}(y,x)$ evaluates the softmax of $-|\langle\psi_Y(y),\psi_X(x)\rangle+b(y)|^2$ over the domain $X$.
In analogy to cross-attention, the maps $\psi_Y,\psi_X,b$ \emph{tokenize} the domains $Y,X$.
For a given \emph{query} $y\in Y$, which is related to a distribution of \emph{keys} $x\in X$ through the tokenization, we take a weighted linear combination of the corresponding \emph{values} $u(x)\in\mathbb{R}$.
Keep in mind that the values are not a linear transformation of the keys, unlike the typical application of attention mechanisms for transformers~\cite{Vaswani2017AttentionIA}.
Based on this similarity, we define the \emph{softmax integral kernel} of a function $f:Y\times X\to\mathbb{R}$ as
\begin{equation}\label{eq:softmax-integral-kernel}
    S^{\lambda}u(y) =
    w(y)
    \frac{\int_{X}
    e^{\lambda f(y,x)}%
    u(x)d\nu(x)}
    {\int_{X}e^{\lambda f(y,x)}d\nu(x)},
\end{equation}
where $w(y)$ is a weight function as before.
Of course, $f(y,x)$ can be factorized in the same manner as the levelset integral kernel.
In analogy to \cref{lemma:gauss-approximation}, one can check that for generic smooth $f$, the softmax integral kernel asymptotically defines a double fibration transform of codimension $n$, \textit{i.e.}, a point-to-point relation.
\end{remark}

\subsection{Universality}

Motivated by \textit{desideratum}~\ref{item:universal}, we consider the question of what \emph{kinds} of operators can be represented by factorized singular integral kernels.
The following lemma indicates that this factorization does not diminish the class of possible functions $f$:
\begin{lemma}[Universal Approximation]\label{lemma:universal}
    Assume that $Y$ and $X$ are compact subsets of Euclidean space, possibly of different dimension.
    Let $\Psi_Y,\Psi_X$ be classes of real-valued functions on $Y,X$ respectively such that $\Psi_Y$ is a dense subset of $C^1(Y;\mathbb{R}^{m})$ and $\Psi_X$ is a dense subset of $C^1(X;\mathbb{R}^{m})$.

    Let $f\in C^2(Y\times X;\mathbb{R}^{m})$ satisfying \cref{assump:nontrivial-levelset,assump:levelset-gradient} be given, and denote the corresponding levelset integral kernel by $R^{\lambda}$.
    Then, for any $\epsilon>0$, there exists $\psi_Y\in\Psi_Y^d,\psi_X\in\Psi_X^d$ and a weight function $\hat{w}:Y\to\mathbb{R}$ defining a factorized levelset integral kernel $L^{\lambda}$ such that for all $u\in C(X)$,
    \begin{equation*}
        |R^{\lambda}u(y) - L^{\lambda}u(y)|
        \lesssim\lambda^{(n''+1)/2}\epsilon\cdot\|u\|_{\infty},
    \end{equation*}
    with $d=O((\epsilon/\|f\|_{C^2})^{-\max\{N,n\}/2})$.
\end{lemma}
Let $f$ be defined based on a double fibration transform $R$ satisfying \cref{assump:levelset}.
Combined with \cref{lemma:gauss-approximation}, this implies that we can attain
\begin{equation*}
    \left|Ru(y)-L^{\lambda}u(y)\right| 
    \leq 
    \left(C\lambda^{-n''/2}+\lambda^{(n''+1)/2}\epsilon\right)\cdot\|u\|_{C^1(X)},
\end{equation*}
where $L^{\lambda}$ is a factorized levelset integral kernel constructed according to \cref{lemma:discrete-kernel}.
In an operator learning context, for large $\lambda$ and sufficiently expressive parameterization of $L^{\lambda}$, optimizing $L^{\lambda}$ to fit a set of examples $\{u_j,Ru_j\}_{j=1}^J$ is a viable approach to learn a double fibration transform from data.

\begin{remark}
    The $d=O(\epsilon^{-\max\{N,n\}/2})$ rate in \cref{lemma:discrete-kernel} is based on a Stone-Weierstra{\ss} approximation of $f$, which yields the bound via a Jackson-type result~\cite[Theorem~1]{bagby2002multivariate}.
    Given the empirical success of neural networks for efficient function approximation, we expect $d$ to be much smaller in practice.
\end{remark}

\section{Stability under Discretization}

In practice, it is unrealistic to assume that we have access to continuous realizations of the example functions $\{u_j,Ru_j\}_{j=1}^J$.
In the setting where we only have access to samples of the functions modeled by $S_1u_j\in\mathbb{R}^{n_1}, S_2Ru_j\in\mathbb{R}^{n_2}$, learning with large $\lambda$ may be impossible, particularly when the fiber $G_y$ intersects none of the sampled points.
We now consider the interplay between the smoothing of the integral operator determined by finite values of $\lambda$ and the discretization of the training data.
Observe that the evaluation functional for $L^{\lambda}u(y)$ described in \eqref{eq:levelset-kernel} does not depend on particular discretizations of the measurement domain $Y$, so we restrict our attention to the discretization of the target domain $X$.

Replace the measure $\nu$ on $X$ with an arbitrary probability measure $\nu_0$, yielding the related operator
\begin{equation}\label{eq:measure-kernel}
    L_{\nu_0}^\lambda u(y)
    =
    w(y)
    \frac{\int_{X}e^{-\lambda |f(y,x)|^2}u(x)d\nu_0(x)}%
    {\int_{X}e^{-\lambda |f(y,x)|^2}d\nu_0(x)}.
\end{equation}
When $\nu_0$ is the uniform probability measure on a discrete subset $X_0\subset X$, this yields
\begin{equation}\label{eq:discrete-kernel}
    L_{X_0}^\lambda u(y)
    =
    w(y)
    \frac{\sum_{x\in X_0} e^{-\lambda |f(y,x)|^2}u(x)}%
    {\sum_{x\in X_0} e^{-\lambda |f(y,x)|^2}},
\end{equation}
which is amenable to evaluation on a computer.
As the discretization becomes finer, $L^{\lambda}_{X_0}$ approximates $L^{\lambda}$:
\begin{lemma}[Discretization]\label{lemma:discrete-kernel}
    Let $y\in Y,u:X\to\mathbb{R}$, and $f:Y\times X\to\mathbb{R}^{m}$ satisfying \cref{assump:nontrivial-levelset,assump:levelset-gradient} be given.
    Assume that $u$ and $f(y,\cdot)$ are Lipschitz on $X$ with respective Lipschitz constants $K_u,K_f$.
    Assume without loss of generality that $\|u\|_{L^\infty}\leq 1$ and $|w(y)|\leq 1$.
    Put
    \begin{equation*}
        c_\lambda = \int_X e^{-\lambda|f(y,x)|^2}d\nu(x).
    \end{equation*}
    Then, for any probability measure $\nu_0$ on $X$ such that 
    \begin{equation*}
        W_1(\nu,\nu_0)\leq \frac{c_\lambda}{2(\sqrt{\lambda}K_f+K_u)},
    \end{equation*}
    we have
    \begin{equation*}
        |L^{\lambda}_{\nu}u(y) - L^{\lambda}_{\nu_0}u(y)|
        \leq
        \frac{2\sqrt{\lambda}K_f+K_u}{c_\lambda^2/2}W_1(\nu,\nu_0),      
    \end{equation*}
    where $W_1$ denotes the $1$-Wasserstein distance.
\end{lemma}
Note that $c_\lambda\asymp\lambda^{-n''/2}$.
This result provides a local Lipschitz constant for $L^{\lambda}_{\nu}$ in terms of the measure $\nu$, indicating that the uniform measure on a suitably fine finite subset $X_0\subset X$ yields a suitable proxy for the true continuous operator, \textit{i.e.}, $L^{\lambda}_{X_0}u\approx L^{\lambda}u$.
That is to say, the levelset integral kernel formulation fulfills \textit{desideratum}~\ref{item:discrete}.
\begin{takeaway}\label{takeaway:second}
    Smoothed approximations of geometric integral operators are learnable from discretized data.
\end{takeaway}

\def\trialcol{001}
\def\mcol{128} 
\begin{figure*}
    \centering
    \begin{minipage}{10cm}
    \resizebox{\linewidth}{!}{%
    \def\datadir{figs/data/fig_radon2d}

\begin{tikzpicture}
    
    \begin{groupplot}[
    group style={
    group size=1 by 3,
    group name=myGroup,
    vertical sep=0.15cm,
    },
    width=8cm,
    height=3.5cm,
    xmin=14,
    xmode=log,
    ymin=3e-4,
    ymax=2,
    ymode=log,
    log basis x=10,
    log basis y=10,
    xtick={16,32,64,128},
    log ticks with fixed point,
    ylabel={Rel. MSE},
    legend columns=4,
    legend style={
        at={(0.0,1.05)},
        anchor=south west,
        font={\footnotesize},
        cells={align=left}
    }
    ]

    \def\ncol{n32}

    \nextgroupplot[xticklabels={,}]

    \def\testcol{matern}

    \addplot[color=black,mark=none,thick] 
    table[col sep=comma,
    x=num_samples,y=\testcol_pinv_\ncol] 
    {\datadir/matern.csv};
    \addlegendentry{MPP}

    \addplot[color=pastelOrange,mark=square*,loosely dotted,thick,mark options={solid}]
    table[col sep=comma,
    x=num_samples,y=\testcol_\ncol] 
    {\datadir/matern_softmax.csv};
    \addlegendentry{Softmax}
    
    \addplot[color=pastelBlue,mark=*,solid,mark options={solid}] 
    table[col sep=comma,
    x=num_samples,y=\testcol_\ncol] 
    {\datadir/matern_fulldim.csv};
    \addlegendentry{$m=2$}

    \addplot[color=pastelBlue,mark=triangle*,dashed,mark options={solid}]
    table[col sep=comma,
    x=num_samples,y=\testcol_\ncol] 
    {\datadir/matern.csv};
    \addlegendentry{$m=1$}


    \begin{pgfonlayer}{front}
    \node [fill=white,fill opacity=0.8,text opacity=1,anchor=north east] at (rel axis cs: 0.98,1) {\textsf{A1} (ID)};
    \end{pgfonlayer}

    \nextgroupplot[xlabel={Training samples}]

    \def\ncol{n32}
    \def\testcol{gauss}

    \addplot[color=black,mark=none,thick] 
    table[col sep=comma,
    x=num_samples,y=\testcol_pinv_\ncol] 
    {\datadir/matern.csv};

    \addplot[color=pastelOrange,mark=square*,loosely dotted,thick,mark options={solid}]
    table[col sep=comma,
    x=num_samples,y=\testcol_\ncol] 
    {\datadir/matern_softmax.csv};
    
    \addplot[color=pastelBlue,mark=*,solid,mark options={solid}] 
    table[col sep=comma,
    x=num_samples,y=\testcol_\ncol] 
    {\datadir/matern_fulldim.csv};

    \addplot[color=pastelBlue,mark=triangle*,dashed,mark options={solid}]
    table[col sep=comma,
    x=num_samples,y=\testcol_\ncol] 
    {\datadir/matern.csv};

    \begin{pgfonlayer}{front}
    \node [fill=white,fill opacity=0.8,text opacity=1,anchor=north east] at (rel axis cs: 0.98,1) {\textsf{A2} (OOD)};
    \end{pgfonlayer}
 
    \end{groupplot}
\end{tikzpicture}

\begin{tikzpicture}
    \begin{groupplot}[
    group style={
    group size=1 by 3,
    group name=myGroup,
    vertical sep=0.15cm,
    },
    tick style={draw=none},
    xticklabel=\empty,
    yticklabel=\empty,
    width=6.5cm,
    height=3.25cm,
    xmin=-1.9,
    xmax=1.9,
    ymin=-0.9,
    ymax=0.9,
    ]

    \ifdefined\mcol
    \else\def\mcol{32}
    \fi

    \ifdefined\trialcol
    \else\def\trialcol{001}
    \fi

    \nextgroupplot[
    axis line style={black,thin},
    ]
    \addplot graphics[xmin=-2,xmax=2,ymin=-1,ymax=1] 
    {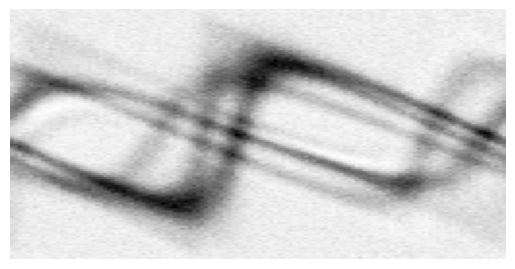};
    \node [fill=white, fill opacity=0.8, text opacity=1, anchor=north east] at (rel axis cs: 0.98,0.98) {\textsf{B1} (MPP)};

    \nextgroupplot[
    axis line style={black,thin},
    ]
    \addplot graphics[xmin=-2,xmax=2,ymin=-1,ymax=1] 
    {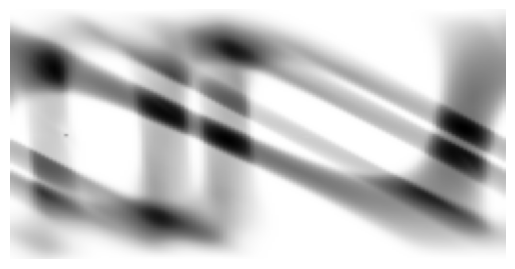};
    \node [fill=white, fill opacity=0.8, text opacity=1, anchor=north east] at (rel axis cs: 0.98,0.98) {\textsf{B2} ($m=2$)};

    \nextgroupplot[
    axis line style={black,thin},
    ]
    \addplot graphics[xmin=-2,xmax=2,ymin=-1,ymax=1] 
    {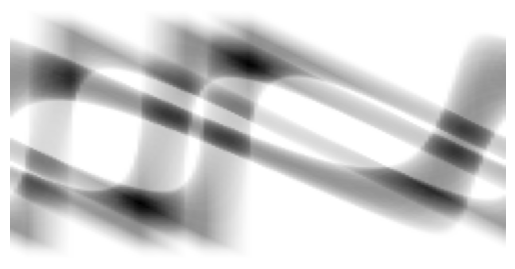};
    \node [fill=white, fill opacity=0.8, text opacity=1, anchor=north east] at (rel axis cs: 0.98,0.98) {\textsf{B3}};

    \end{groupplot}
\end{tikzpicture}
    }
    \end{minipage}
    \hspace{-0.35cm}
    \begin{minipage}{5cm}
    \resizebox{\linewidth}{!}{%
    \def\datadir{figs/data/fig_radon2d}

\begin{tikzpicture}
    \begin{groupplot}[
    group style={
    group size=2 by 1,
    group name=myGroup,
    horizontal sep=0.25cm,
    },
    tick style={draw=none},
    xticklabel=\empty,
    yticklabel=\empty,
    width=4cm,
    height=4cm,
    xmin=-1, xmax=1,
    ymin=-1, ymax=1,
    clip mode=individual
    ]

    \nextgroupplot[axis line style={draw=none}]
    \addplot graphics[xmin=-1,xmax=1,ymin=-1,ymax=1] 
    {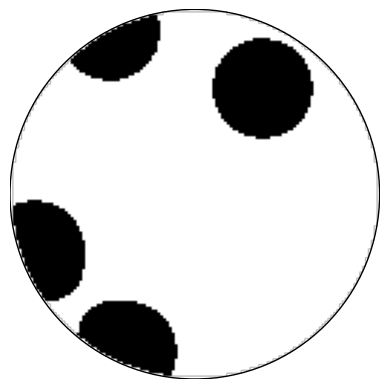};
    \node [fill=white, fill opacity=0.8, text opacity=1, anchor=north east] at (rel axis cs: 1,1) {\large\textsf{B4}};

    \nextgroupplot[axis line style={draw=none}]
    \addplot graphics[xmin=-1,xmax=1,ymin=-1,ymax=1] 
    {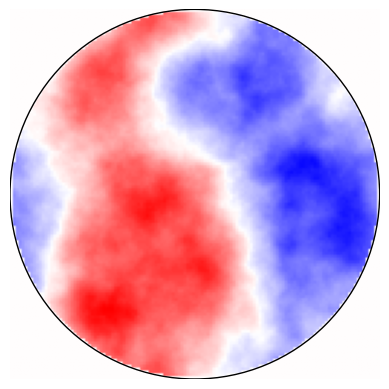};
    \node [fill=white, fill opacity=0.8, text opacity=1, anchor=north east] at (rel axis cs: 1,1) {\large\textsf{C1}};

    \end{groupplot}
\end{tikzpicture}
    }
    \resizebox{\linewidth}{!}{%
    \def\datadir{figs/data/fig_radon2d}

\begin{tikzpicture}
    \begin{groupplot}[
    group style={
    group size=2 by 1,
    group name=myGroup,
    horizontal sep=0.25cm,
    },
    tick style={draw=none},
    xticklabel=\empty,
    yticklabel=\empty,
    width=4cm,
    height=4cm,
    xmin=-1, xmax=1,
    ymin=-1, ymax=1,
    clip mode=individual
    ]

    \nextgroupplot[width=7cm,xmin=-1.9,xmax=1.9,ymin=-0.9,ymax=0.9]
    \addplot graphics[xmin=-2,xmax=2,ymin=-1,ymax=1] 
    {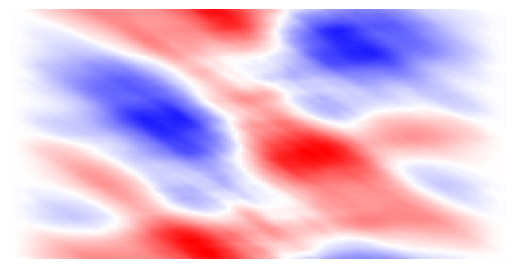};
    \node [fill=white, fill opacity=0.8, text opacity=1, anchor=north east] at (rel axis cs: 1,1) {\large\textsf{C2}};
    
    \nextgroupplot[width=2cm, axis line style={draw=none}]
    \addplot graphics[xmin=-1,xmax=1,ymin=-1,ymax=1] 
    {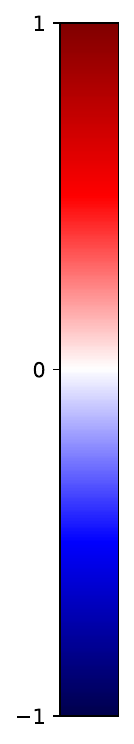};

    \end{groupplot}
\end{tikzpicture}
    }
    \end{minipage}
    \caption{%
    Levelset integral kernels can learn the Radon transform from few samples.
    \textbf{(A)} Relative MSE of learned operators when evaluated on ID \textit{(A1)} and OOD \textit{(A2)} data, averaged over five training runs.
    Levelset integral kernels exhibit superalgebraic test error decay in the small-sample ($J\leq 32$) regime, before performance saturates due to discretization.
    \textbf{(B)} Example operators applied to OOD test image $u$, pictured in \textit{(B4)}.
    \textit{(B1)} Estimated transform by MPP from $J=\mcol$ samples.
    \textit{(B2)} Estimated transform $L^{\lambda}u$ by levelset integral kernel trained on $J=\mcol$ samples with $m=2$.
    \textit{(B3)} Ground truth Radon transform $Ru$.
    \textbf{(C)} Example training data pair $(u_j,Ru_j)$ where $u_j$ is sampled from a Mat{\'e}rn random field.
    The levelset methods outperform the others both quantitatively and qualitatively.
    }
    \label{fig:learning-radon2d}
\end{figure*}
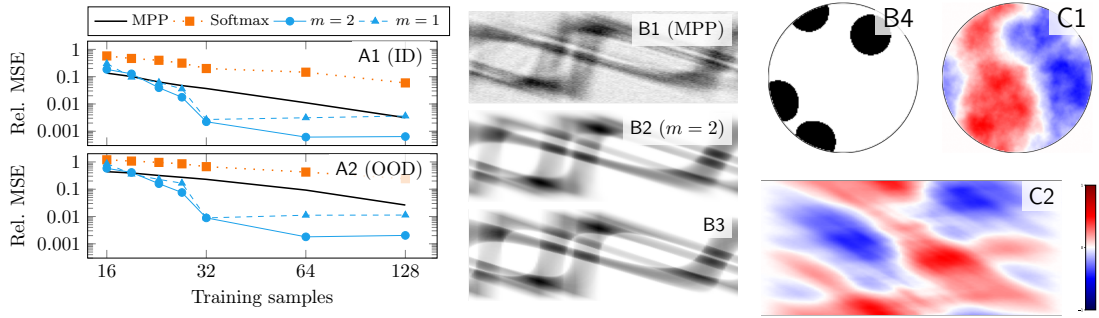

\section{Learning Radon and Ray Transforms}

We demonstrate the advantage of faithfully parameterizing the geometric structure of double fibration transforms in the operator learning problem for the Radon transform~(\cref{example:radon}).
We consider examples $\{S_1u_j,S_2Ru_j\}_{j=1}^J$ where each $u_j$ is drawn independently from a Mat\'{e}rn random field, pictured in \cref{fig:learning-radon2d}~(C).
The operators $S_1,S_2$ sample their respective functions on $32\times 32$ discretizations of the domains $Y,X$.
The levelset integral kernels are trained to minimize the mean-squared error (MSE) relative to the training data.

Levelset integral kernels with $m\in\{1,2\}$ are compared to the Moore-Penrose pseudoinverse (MPP) and the softmax integral kernel (see~\cref{remark:attention}).
Testing is done on independently drawn data from the same distribution as the training data (In-distribution, ID), and on randomly distributed gaussian functions on $X$ (Out-of-distribution, OOD).
As shown in \cref{fig:learning-radon2d}~(A), the levelset methods exhibit superalgebraically decaying test error in the small-sample regime ($J\leq 32$) before the test error saturates due to discretization.
This suggests that the levelset methods are able to match the statistical guarantees of \cref{theorem:data-efficiency}.
Moreover, there is only modest degradation when testing on OOD data.
The MPP is incapable of modeling structures outside of the span of the training set, and the softmax kernel does not have a strong bias towards the correct geometric structure.
In \cref{fig:learning-radon2d}~(B), we examine the estimated transforms on a blob-like image, and note that the levelset parameterization more accurately realizes the geometric structure of the target image.

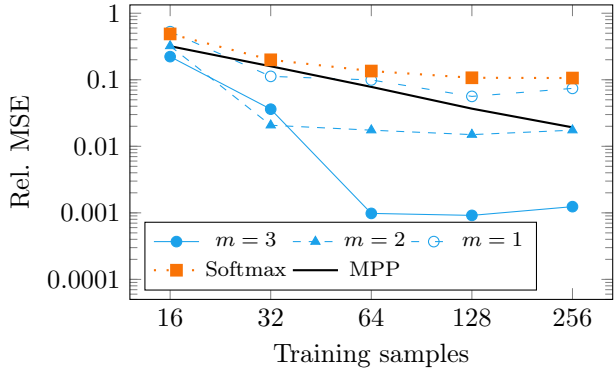
\begin{figure}
    \centering
    \def\datadir{figs/data/fig_ray3d}
\begin{tikzpicture}

\pgfplotsset{cycle list name=threelines}
    
    \begin{axis}[
    width=8cm,
    height=5.5cm,
    xmode=log,
    ymin=5e-5,
    ymode=log,
    log basis x=10,
    log basis y=10,
    xtick={16,32,64,128,256,512},
    log ticks with fixed point,
    xlabel={Training samples},
    ylabel={Rel. MSE},
    legend columns=3,
    legend pos={south west},
    legend style={
        font={\footnotesize},
        cells={align=left},
    },
    ]

    \addplot[color=pastelBlue,mark=*]
    table[col sep=comma,x=num_samples,y=matern_n32] 
    {\datadir/plane_fulldim.csv}; 
    \addlegendentry{$m=3$}

    \addplot[color=pastelBlue,mark=triangle*,dashed,mark options={solid}]
    table[col sep=comma,x=num_samples,y=matern_n32] 
    {\datadir/plane.csv}; 
    \addlegendentry{$m=2$}

    \addplot[color=pastelBlue,mark=o,loosely dashed,mark options={solid}]
    table[col sep=comma,x=num_samples,y=matern_n32] 
    {\datadir/plane_codim.csv}; 
    \addlegendentry{$m=1$}

    \addplot[color=pastelOrange,mark=square*,loosely dotted,thick,mark options={solid}]
    table[col sep=comma,x=num_samples,y=matern_n32] 
    {\datadir/plane_softmax.csv}; 
    \addlegendentry{Softmax}

    \addplot[color=black,mark=none,thick]
    table[col sep=comma,x=num_samples,y=matern_pinv_n32] 
    {\datadir/plane.csv}; 
    \addlegendentry{MPP}

    \end{axis}
\end{tikzpicture}
    \caption{%
    Relative MSE of learned operators for approximating the Euclidean ray transform in $\mathbb{R}^3$, averaged over five training runs.
    Overspecification of the codimension ($m=3$) of the integral transform yields superior performance to more constrained methods, especially when compared to the underspecified codimension ($m=1$).
    }
    \label{fig:ray3d}
\end{figure}

We carry out a similar experiment for the Euclidean ray transform in $\mathbb{R}^3$ as defined in \cref{example:ray}, which is a double fibration transform of codimension $n''=2$.
Using $32\times 32\times 32$ discretizations of the domains $Y,X$, we compare the MPP, softmax, and levelset integral kernel with maximum codimension ranging over $m\in\{1,2,3\}$.
Training and testing is done on independent instances of a Mat\'{e}rn random field.
For $J\in[16,256]$ samples, the overspecified codimension ($m=3$) levelset integral kernel outperforms all other methods, followed by the exactly specified codimension ($m=2$), as shown in \cref{fig:ray3d}.
We see that \emph{underspecifying} the codimension by choosing $m<n''$ yields poor performance, comparable to the softmax integral kernel.
We conjecture that the improvement in performance by overspecifying the codimension is due to extra flexibility in choosing the best coordinate system with which to locally represent the incidence submanifold, which may vary due to the curvature of $Z$ as a submanifold of $Y\times X$.

\subsection{Empirically Determining the Codimension}

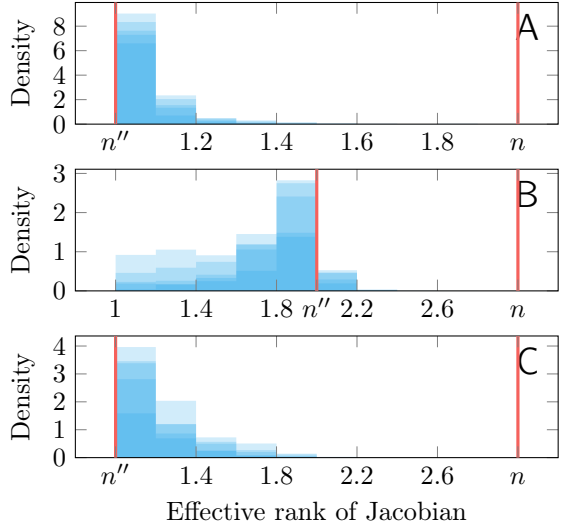
\begin{figure}
    \centering
    \begin{tikzpicture}
    \begin{groupplot}[
    group style={
    group size=1 by 3,
    group name=myGroup,
    vertical sep=0.6cm,
    },
    %
    %
    width=8cm,
    height=3.2cm,
    ymin=0,
    major grid style={very thick, pastelRed},
    axis on top=true,
    ]

    \nextgroupplot[xmin=0.9, xmax=2.1,
    ylabel={Density},
    xtick={1.2, 1.4, 1.6, 1.8},
    extra x ticks={1, 2},
    extra x tick style={%
        text height=height("0"),
        grid=major,
    },
    extra x tick labels={
        $n''$,
        $n$
    }]

    \def\datadir{figs/data/fig_radon2d}
    
    \addplot+[ybar interval, mark=no, draw=none, fill=pastelBlue, fill opacity=0.2] 
    table[col sep=comma,x=effrank,y=density]
    {\datadir/rankhist_fulldim_001.csv};
    \addplot+[ybar interval, mark=no, draw=none, fill=pastelBlue, fill opacity=0.2] 
    table[col sep=comma,x=effrank,y=density]
    {\datadir/rankhist_fulldim_002.csv};
    \addplot+[ybar interval, mark=no, draw=none, fill=pastelBlue, fill opacity=0.2] 
    table[col sep=comma,x=effrank,y=density]
    {\datadir/rankhist_fulldim_003.csv};
    \addplot+[ybar interval, mark=no, draw=none, fill=pastelBlue, fill opacity=0.2] 
    table[col sep=comma,x=effrank,y=density]
    {\datadir/rankhist_fulldim_004.csv};
    \addplot+[ybar interval, mark=no, draw=none, fill=pastelBlue, fill opacity=0.2] 
    table[col sep=comma,x=effrank,y=density]
    {\datadir/rankhist_fulldim_005.csv};
    
    \node [fill=white, fill opacity=0.8, text opacity=1, anchor=north east] at (rel axis cs: 0.98,0.98) {\Large\textsf{A}};

    \nextgroupplot[xmin=0.8, xmax=3.2, 
    ylabel={Density},
    xtick={1, 1.4, 1.8, 2.2, 2.6},
    extra x ticks={2, 3},
    extra x tick style={%
        text height=height("0"),
        grid=major,
    },
    extra x tick labels={
        $n''$,
        $n$
    }]

    \def\datadir{figs/data/fig_ray3d}
    
    \addplot+[ybar interval, mark=no, draw=none, fill=pastelBlue, fill opacity=0.2] 
    table[col sep=comma,x=effrank,y=density]
    {\datadir/rankhist_fulldim_001.csv};
    \addplot+[ybar interval, mark=no, draw=none, fill=pastelBlue, fill opacity=0.2] 
    table[col sep=comma,x=effrank,y=density]
    {\datadir/rankhist_fulldim_002.csv};
    \addplot+[ybar interval, mark=no, draw=none, fill=pastelBlue, fill opacity=0.2] 
    table[col sep=comma,x=effrank,y=density]
    {\datadir/rankhist_fulldim_003.csv};
    \addplot+[ybar interval, mark=no, draw=none, fill=pastelBlue, fill opacity=0.2] 
    table[col sep=comma,x=effrank,y=density]
    {\datadir/rankhist_fulldim_004.csv};
    \addplot+[ybar interval, mark=no, draw=none, fill=pastelBlue, fill opacity=0.2] 
    table[col sep=comma,x=effrank,y=density]
    {\datadir/rankhist_fulldim_005.csv};
    
    \node [fill=white, fill opacity=0.8, text opacity=1, anchor=north east] at (rel axis cs: 0.98,0.98) {\Large\textsf{B}};

    \nextgroupplot[xmin=0.8, xmax=3.2, 
    xlabel={Effective rank of Jacobian},
    ylabel={Density},
    xtick={1.4, 1.8, 2.2, 2.6},
    extra x ticks={1, 3},
    extra x tick style={%
        text height=height("0"),
        grid=major,
    },
    extra x tick labels={
        $n''$,
        $n$
    }]

    \def\datadir{figs/data/fig_radon3d}
    
    \addplot+[ybar interval, mark=no, draw=none, fill=pastelBlue, fill opacity=0.2] 
    table[col sep=comma,x=effrank,y=density]
    {\datadir/rankhist_fulldim_001.csv};
    \addplot+[ybar interval, mark=no, draw=none, fill=pastelBlue, fill opacity=0.2] 
    table[col sep=comma,x=effrank,y=density]
    {\datadir/rankhist_fulldim_002.csv};
    \addplot+[ybar interval, mark=no, draw=none, fill=pastelBlue, fill opacity=0.2] 
    table[col sep=comma,x=effrank,y=density]
    {\datadir/rankhist_fulldim_003.csv};
    \addplot+[ybar interval, mark=no, draw=none, fill=pastelBlue, fill opacity=0.2] 
    table[col sep=comma,x=effrank,y=density]
    {\datadir/rankhist_fulldim_004.csv};
    \addplot+[ybar interval, mark=no, draw=none, fill=pastelBlue, fill opacity=0.2] 
    table[col sep=comma,x=effrank,y=density]
    {\datadir/rankhist_fulldim_005.csv};
    
    \node [fill=white, fill opacity=0.8, text opacity=1, anchor=north east] at (rel axis cs: 0.98,0.98) {\Large\textsf{C}};

    \end{groupplot}
\end{tikzpicture}
    \caption{%
    The effective rank of the Jacobian evaluated at incident points indicates the codimension of the underlying geometric relation.
    Distribution of the effective Jacobian rank for learned approximation to the \textbf{(A)} Radon transform in $\mathbb{R}^2$ from $J=128$ training samples, \textbf{(B)} Euclidean ray transform from $J=256$ training samples, and \textbf{(C)} Radon transform in $\mathbb{R}^3$ from $J=256$ training samples.
    All histograms overlay the empirical distributions over five training runs, with the true underlying codimension $n''$ and the dimension $n$ of the domain marked.
    }
    \label{fig:rankhists}
\end{figure}

As shown in \cref{fig:learning-radon2d,fig:ray3d}, the more flexible models where $m=n$ appear to outperform those with the ``correct'' codimension $m=n''$.
Despite the apparent performance advantages, the regime of $n>n''$ makes it difficult to guarantee that $f$ will satisfy \cref{assump:levelset-gradient} with the correct value of $n''$.
Indeed, \cref{assump:levelset-gradient} holds for generic smooth $f$ only when $m=n''$.

Still, a reasonable estimate of the true underlying geometry is expected to have an \emph{approximate} low-rank structure in accordance with the true underlying geometry.
We evaluate this for the trained approximation of the Radon and ray transforms by examining the spectrum of the Jacobian over the incident points.
To do so, we evaluate the integral kernel over a discretization of $X$.
For $64$ points $y\in Y$ selected at random, we sample $64$ points $x\in X$ such that $|f(y,x)|^2$ is minimized (under \cref{assump:nontrivial-levelset}, the zero-levelset).
We then compute the effective rank of the Jacobian at these points, defined for a matrix $A\in\mathbb{R}^{n\times m}$ as $\effrank(A):=\|A\|_*/\|A\|_2$, that is, the ratio of the nuclear norm to the operator norm.
Observe that $\effrank(A)\in[1,m]$ for $m\leq n$.

We demonstrate the concentration of the effective rank of the Jacobian for the levelset integral kernels trained on the maximal number of samples ($J=128$ for the Radon transform, $J=256$ for the ray transform) in \cref{fig:rankhists}.
Additionally, we train a levelset integral kernel on $J=256$ samples for the Radon transform in $\mathbb{R}^3$, which is a double fibration transform of codimension $n''=1$ defined by integration over planes.
In \cref{fig:rankhists}~(A,C), the effective rank of the Jacobian evaluated at incident points concentrates near one, which is the codimension of the true operators.
Similarly, in \cref{fig:rankhists}~(B), the effective rank concentrates near two.

\section{Application: Riemannian Inverse Problems}

By \cref{lemma:gauss-approximation}, a levelset integral kernel can be inspected to determine the parameters of the geometric integral operator in the limit as $\lambda\to\infty$: namely, the incidence relation $Z\subset Y\times X$ and the amplitude $a(y,x)$.
To demonstrate, we consider the inverse problem of determining an isotropic Riemannian metric given a learned approximation of the corresponding geodesic ray transform~\cite{monard2014,HolmanUhlmann}.

Let $X\subset\mathbb{R}^d$ be a bounded open set with smooth boundary, and denote the inward-pointing cosphere bundle on $X$ by $\partial_{in}S^*X$.
Let $c:X\to\mathbb{R}^{>0}$ be a smooth \emph{wavespeed} on the domain, which defines a Riemannian metric $g(x)=c^{-2}(x)$.
For initial condition $y=(x,\xi)\in \partial_{in}S^*X$, we denote the corresponding geodesic curve by $\gamma_y(t):(0,\tau_y)\to X$, where $\tau_y$ is the exit time of the geodesic.
For simplicity, we assume the wavespeed yields a \emph{nontrapping metric}~\cite{monard2014}, so that $\tau_y<\infty$ for all $y$.
This defines the \emph{geodesic ray transform}
\begin{equation}\label{eq:geodesic-ray-transform}
    Ru(y) = \int_{0}^{\tau_y}u(\gamma_y(t))dt.
\end{equation}
The geodesic ray transform is a double fibration transform~\cite{mazzucchelli2023}, where the incidence relation is given by pairs $(y,x)$ such that $x=\gamma_y(t)$ for some $t\in(0,\tau_y)$.
Thus, the fibers $G_y$ are given by the geodesic curves $\gamma_y((0,\tau_y))$, and the measure $\mu_y(x)$ (and thus, the amplitude $a(y,x)$) is inversely proportional to the wavespeed $c(x)$.

\begin{example}[Lensing Ray Transform]\label{example:lens}
    Let $X$ be the unit disc in $\mathbb{R}^2$, and let $Y=\partial_{in}S^*X$ be the inward-pointing cosphere bundle on $\partial X=S^1$.
    Let $R:\mathcal{D}'(X)\to\mathcal{D}'(Y)$ be the geodesic ray transform for a Riemannian metric determined by a lensing metric~\cite{monard2014}:
    \begin{equation*}
        c(x) = 
        \exp\left(
            -\frac{k}{2}
            \exp\left(
                -\frac{|x|^2}{2\sigma^2}
            \right)
        \right),
    \end{equation*}
    where $k,\sigma>0$ are fixed parameters.
    The geodesics $\gamma_y(t)$ ``lens'' around the origin of the disc $X$, with amplitude $a(y,x)$ that is small near the boundary and large near the origin.
\end{example}

\begin{figure}
    \centering
    \begin{minipage}[c]{0.54\linewidth}
      \resizebox{\linewidth}{!}{\def\graphicdir{figs/data/fig_lensray2d}

\begin{tikzpicture}
    \begin{groupplot}[
    group style={
    group size=3 by 2,
    group name=myGroup,
    horizontal sep=0.1cm,
    vertical sep=0.3cm,
    },
    tick style={draw=none},
    xticklabel=\empty,
    yticklabel=\empty,
    width=3.5cm,
    height=3.5cm,
    xmin=-1,
    xmax=1,
    ymin=-1,
    ymax=1,
    clip mode=individual
    ]

    \nextgroupplot[width=2cm, axis line style={draw=none}]
    \def\svgwidth{0.35cm}
    \node[/pgfplots/plot graphics/node, anchor=east] at (rel axis cs: 1,0.5) 
    {\import{\graphicdir/512}{logit_bar_pdf.tex}};

    \nextgroupplot[width=6.125cm,xmin=-1.9,xmax=1.9,ymin=-0.9,ymax=0.9]
    \addplot graphics[xmin=-2,xmax=2,ymin=-1,ymax=1] 
    {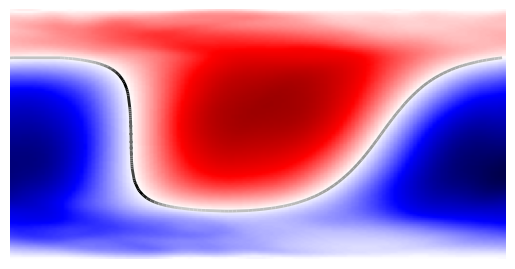};
    \node [fill=white, fill opacity=0.8, text opacity=1, anchor=north west] at (rel axis cs: 0,1) {\large\textsf{A1}};

    \nextgroupplot[axis line style={draw=none}]
    \addplot graphics[xmin=-1,xmax=1,ymin=-1,ymax=1] 
    {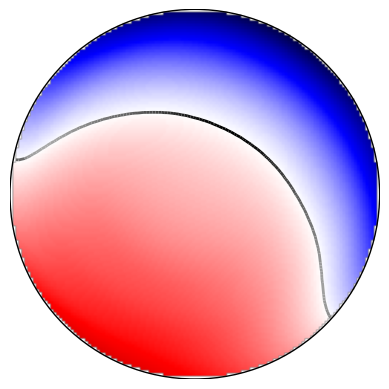};
    \node [fill=white, fill opacity=0.8, text opacity=1, anchor=north west] at (rel axis cs: 0,1) {\large\textsf{A2}};

    \nextgroupplot[width=2cm, axis line style={draw=none}]
    \def\svgwidth{0.3cm}
    \node[/pgfplots/plot graphics/node, anchor=east] at (rel axis cs: 1,0.5) 
    {\import{\graphicdir/512}{image_bar_pdf.tex}};
     
    \nextgroupplot[width=6.125cm,xmin=-1.9,xmax=1.9,ymin=-0.9,ymax=0.9]
    \addplot graphics[xmin=-2,xmax=2,ymin=-1,ymax=1] 
    {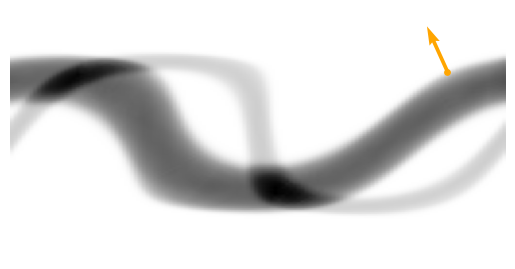};
    \node [fill=white, fill opacity=0.8, text opacity=1, anchor=north west] at (rel axis cs: 0,1) {\large\textsf{B1}};

    \nextgroupplot[axis line style={draw=none}]
    \addplot graphics[xmin=-1,xmax=1,ymin=-1,ymax=1] 
    {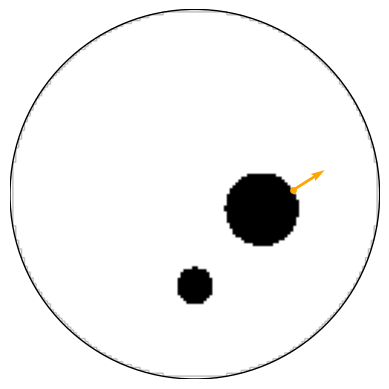};
    \node [fill=white, fill opacity=0.8, text opacity=1, anchor=north west] at (rel axis cs: 0,1) {\large\textsf{B2}};

    \end{groupplot}
\end{tikzpicture}}
    \end{minipage}
    \begin{minipage}[c]{0.43\linewidth}
      \resizebox{\linewidth}{!}{\def\graphicdir{figs/data/fig_lensray2d}

\begin{tikzpicture}
    \begin{groupplot}[
    group style={
    group size=3 by 1,
    group name=myGroup,
    horizontal sep=0.0cm,
    },
    tick style={draw=none},
    xticklabel=\empty,
    yticklabel=\empty,
    axis line style={draw=none},
    height=5cm,
    xmin=-1, xmax=1,
    ymin=-1, ymax=1,
    ]

    \nextgroupplot[width=5cm]
    \addplot graphics[xmin=-1,xmax=1,ymin=-1,ymax=1] 
    {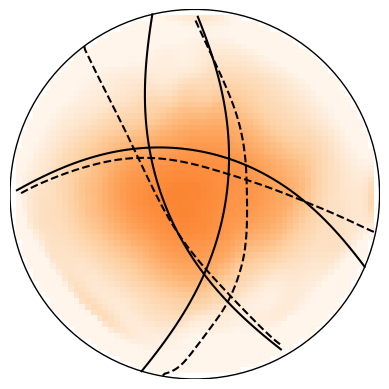};
    \node [fill=white, fill opacity=0.8, text opacity=1, anchor=north west] at (rel axis cs: 0,1) {\large\textsf{C1}};
    \node [fill=white, fill opacity=0.8, text opacity=1, anchor=south west] at (rel axis cs: 0,0) {\large$(J=32)$};

    \nextgroupplot[width=5cm]
    \addplot graphics[xmin=-1,xmax=1,ymin=-1,ymax=1] 
    {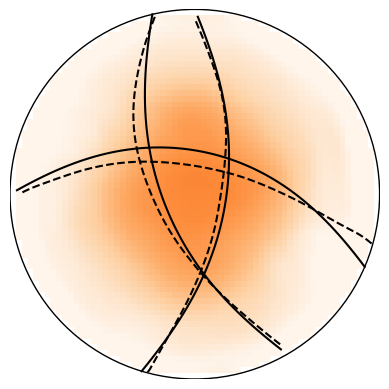};
    \node [fill=white, fill opacity=0.8, text opacity=1, anchor=north west] at (rel axis cs: 0,1) {\large\textsf{C2}};
    \node [fill=white, fill opacity=0.8, text opacity=1, anchor=south west] at (rel axis cs: 0,0) {\large$(J=512)$};

    \nextgroupplot[width=2cm, axis line style={draw=none}]
    \def\svgwidth{0.4cm}
    \node[/pgfplots/plot graphics/node, anchor=west] at (rel axis cs: 0,0.5) 
    {\import{\graphicdir/512}{wavespeed_bar_pdf.tex}};
    
    \end{groupplot}
\end{tikzpicture}}
    \end{minipage}
    \caption{%
    Geometry of a learned geodesic ray transform.
    %
    \textbf{(A)} Levelset function $f(y,x)$ for fixed $x\in X$ \textit{(A1)} and $y\in Y$ \textit{(A2)}.
    The fibers $H_x,G_y$ are shown, with amplitude $a(y,x)$ indicated by the shade (darker$=$larger).
    \textbf{(B1)} Function $L^\lambda u$ with $(y,\eta)\in\wf(Lu)$.
    \textbf{(B2)} Function $u$ with $(x,\xi)\in\wf(f)$, such that $(y,-\eta;x,\xi)\in N^*Z$.
    \textbf{(C1,C2)} Estimated wavespeeds for $J=32,512$, with true geodesic rays (solid) and estimated rays with same initial point $y\in Y$ (dashed).
    }
    \label{fig:canonical-relation}
\end{figure}

\subsection{Propagation of Singularities}

In many settings, the most important features of a function $u$ to reconstruct from measurements $Ru$ are the \emph{singularities}.
Specifically, the wavefront set of $u$, which is related to $\wf(Ru)$ by the conormal bundle of $Z$.

The levelset integral kernel parameterization of a double fibration transform makes it straightforward to compute this relation.
The incidence submanifold $Z$ consists of points $(y,x)$ such that $f(y,x) = 0$.
Since $f$ is differentiable, the conormal space at a given $(y,x)\in Z$ is equal to the row span of the Jacobian of $f(y,x)$, under \cref{assump:levelset-gradient}.
We illustrate this in \cref{fig:canonical-relation}~(A,B) for a levelset integral kernel trained on data generated from the lensing ray transform~(\cref{example:lens}), where the Jacobian of $f(y,x)$ dictates the relationship between $\wf(u)$ and $\wf(Lu)$.

\subsection{Wavespeed Recovery}

We now consider the inverse problem of recovering the wavespeed $c(x)$ from a set of examples $\{u_j,Ru_j\}_{j=1}^J$.
We first train a levelset integral kernel on the data to yield an estimated operator $L$, parameterized by some learned function $f(y,x)$.
Since the amplitude of the true operator is independent of $y$, we estimate the wavespeed at an arbitrarily given $x\in X$ by taking the median value of the amplitude $a_f(y,x)$ for a set of sampled points $(y,x)\in f^{-1}(0)$, yielding an estimate of the true underlying amplitude $\hat{a}(x)$. 
The wavespeed $\hat{c}(x)\sim 1/\hat{a}(x)$ follows from this estimate.
We convolve the estimated wavespeed with a gaussian kernel to smooth out numerical artifacts.
The estimated wavespeed is then used to compute a family of corresponding estimated geodesic curves $\hat{\gamma}_y(t)$.

We illustrate the results of this procedure in \cref{fig:canonical-relation}~(C1,C2), for operators trained on $J=32$ and $J=512$.
The general structure of the wavespeed is correctly recovered in both cases: $\hat{c}(x)$ is largest near the boundary, and smallest near the center.
When the number of samples is small, though, the local errors in the estimated wavespeed compound to yield poor estimates of the geodesics; on the other hand, when the number of samples is large, the estimated geodesic curves are fairly close to those arising from the true wavespeed.
This is despite the fact that the parameterization of the operator incorporates no \textit{a priori} knowledge of the geometric properties of a geodesic ray transform, namely, the existence of an underlying Riemannian metric.

\begin{takeaway}\label{takeaway:fourth}
    Levelset integral kernels naturally represent geometric integral operators, including double fibration transforms.
\end{takeaway}

\section{Conclusion}

We have considered the problem of learning an integral operator $R$ from a set of examples $\{u_j,Ru_j\}_{j=1}^J$, focusing on the case where $R$ is a double fibration transform mapping functions on some domain $X$ to functions on another domain $Y$.
These operators are naturally modeled by levelset methods, from which the incidence relation, amplitude, and even codimension can be estimated.
We have demonstrated the theoretical~(\cref{theorem:data-efficiency}) and practical~(\cref{lemma:gauss-approximation,lemma:universal,lemma:discrete-kernel}) properties of learning these transforms from data, complimenting prior works on operator learning that assume the geometry is known \textit{a priori}.
We foresee learnable levelset integral kernels finding use in modeling more complicated integration phenomena, either via incorporation into deep nonlinear models for solving geometric inverse problems, or by composition with pseudodifferential operators for richer scattering phenomena.

\subsection*{Acknowledgements}
TMR and RGB were partially supported by ONR grant N00014-23-1-2714, ONR MURI N00014-20-1-2787, DOE grant DE-SC0020345, and DOI grant 140D0423C0076.
LT was supported by the Australian Research Council under grants DP260103195 and DP220101808.
ID was supported by European Research Council Starting Grant 852821--SWING.
MVdH gratefully acknowledges the support of the Department of Energy BES program under grant DE-SC0020345, Oxy, the corporate members of the Geo-Mathematical Imaging Group at Rice University and the Simons Foundation under the MATH+X Program.

\medskip

\noindent TMR would like to acknowledge his daughter for her assistance in preparing this manuscript, as well as his wife for her assistance in preparing his daughter.

\appendix

\section{Data-efficient Recovery}\label{sec:recovery}

\subsection{Compressibility of Gabor Representations}\label{sec:recovery:compressible}

Suppose $X$ is a Euclidean domain of dimension $n$, and let $\{g_{\hat x,\hat\xi}\}_{(\hat x,\hat\xi)\in\Lambda}$ be a Gabor system on $X$, where $g$ is a $C_0^{\infty}$ window and $\Lambda\subset\mathbb{R}^{2n}$ is a rectangular lattice, such that $\{g_{\hat x,\hat\xi}\}_{(\hat x,\hat\xi)\in\Lambda}$ is a frame for $L^2(\mathbb{R}^n)$.
We denote by $\gamma$ the canonical dual window, which as a Schwartz function by~\cite[Proposition~5.5]{janssen1994duality}, and satisfies
\begin{equation*}
    f = \sum_{(\hat x, \hat \xi)\in\Lambda_X}\langle f, \gamma_{\hat x, \hat \xi}\rangle g_{\hat x, \hat \xi}.
\end{equation*}
for all $f\in L^2(\mathbb{R}^n)$, and in particular $f\in L^2(X)$.
Let $h\in C_0^{\infty}(\mathbb{R}^N)$ denote a window function that is separable as a tensor product $h(y)=\prod_{j=1}^{N}h(y_j)$.

\begin{lemma}
\label{lem: b bound for integral}
Suppose the operator $R:\mathcal{D}'(Y)\to\mathcal{D}'(X)$ has local representation given by the Schwartz kernel
\begin{eqnarray}
    \label{eq: oscillatory kernel}
R(y,x) = a(y,x)\int_{\mathbb R^{n''}} e^{i(b(x,y') - y'')\cdot \zeta''} d\zeta''.
\end{eqnarray}
Then, for any $\ell \in \mathbb N$,
\begin{equation}
    \label{eq: b bound for integral}
    \left|\langle h_{\hat y,\hat \eta}, R g_{\hat x, \hat \xi} \rangle \right| \leq C_{\ell} \llangle  \left(b(\hat x,\hat y'), b_x(\hat x,\hat y')^T \hat \eta'', b_{\hat y}(\hat x,\hat y')^T \hat \eta''\right) - \left(\hat y'', -\hat \xi, -\hat \eta'\right) \rrangle^{-\ell}
\end{equation}
\end{lemma}
\begin{proof}
Observe that the oscillatory integral given in \eqref{eq: oscillatory kernel} is supported only on the set $\{y'' = b(x,y')\}$ so we may replace the weight $a(y,x)$ in \eqref{eq: oscillatory kernel} by $a(y', b(y'), x) := \tilde a(y', x)$.

We compute
{\small
\begin{equation*}
\int_Y \int_{\mathbb R^{n''}} e^{i(b(x,y') - y'') \cdot \zeta''}\tilde a(y', x) h_{\hat y , \hat \eta}(y) dy d\zeta'' =  
\int_{Y}\int_{\mathbb R^{n''}} e^{i(b(x,y'+\hat y')\cdot \zeta''}   e^{iy'' \cdot (\hat\eta'' - \zeta'')} e^{-i\hat y'' \cdot \zeta''} e^{iy' \cdot \hat \eta'}\tilde a(y',x) h(y) dy'' dy' d\zeta''
\end{equation*}
}
Using the tensorial structure of the window function $h(y) = h(y') h(y'')$, see that the $y''$ integral is simply the Fourier transform in the $y''$ variable.
Thus,
{\footnotesize
\begin{equation*}
\int_Y \int_{\mathbb R^{n''}} e^{i(b(x,y') - y'') \cdot \zeta''} h_{\hat y , \hat \eta}(y)\tilde a(y',x) dy d\zeta'' =  \int_{y'\in \mathbb R^{N-n''}} \int_{\zeta'' \in \mathbb R^{n''}} e^{i(b(x,y'+\hat y')\cdot \zeta''}   e^{-i\hat y'' \cdot \zeta''} e^{iy' \cdot \hat \eta'} h(y')\hat h(\zeta'' - \hat\eta'')  \tilde a(y',x)dy' d\zeta''.
\end{equation*}
}
Using the above expression, represent $\langle h_{\hat y, \hat \eta}, R g_{\hat x, \hat \xi}\rangle$ as
\begin{equation}\label{eq: inner product rep}
\left |\langle h_{\hat y, \hat \eta}, R g_{\hat x, \hat \xi}\rangle \right| =
\left| \int_X \int_{y'\in \mathbb R^{N-n''}} \int_{\zeta'' \in \mathbb R^{n''}}e^{i(b(x+\hat x,y'+\hat y')\cdot (\zeta''+\hat \eta'')}   e^{-i\hat y'' \cdot \zeta''} e^{iy' \cdot \hat \eta'} h(y')\hat h(\zeta'') e^{i \hat \xi\cdot x}  g(x)\tilde a(y',x)dy' d\zeta''dx\right|\end{equation}
Take the Taylor expansion
\begin{equation*}
b(x,y'+\hat y') = b(\hat x,\hat y')  + b_x(\hat x, \hat y') x + b_{y'}(x,\hat y') y' + b_2(x,y)
\end{equation*}
and insert this into \eqref{eq: inner product rep}, yielding
\begin{equation}\label{eq:taylor-expanded-form}
\begin{aligned}
\left |\langle h_{\hat y, \hat \eta}, R g_{\hat x, \hat \xi}\rangle \right|
&= \Big| \iiint e^{i(b(\hat x,\hat y') - \hat y '')\cdot \zeta'' + i x \cdot (b_x(\hat x, \hat y')^T\hat \eta'' + \hat \xi)+i y' \cdot( b_{y'}(\hat x, \hat y')^T\hat\eta '' + \hat \eta')} \\
&\phantom{=\Big|\iiint}\ \times e^{ib_2(x,y) \cdot (\zeta'' + \hat \eta'') + i (b_x(\hat x, \hat y') x )\cdot \zeta'' + i(b_{y'} (\hat x, \hat y') y' )\cdot \zeta''} h(y')g(x) \hat h(\zeta'')\tilde a(y',x) \Big|.
\end{aligned}
\end{equation}
Observe that 
{\footnotesize
\begin{equation*}
e^{i(b(\hat x,\hat y') - \hat y '')\cdot \zeta'' +i x \cdot (b_x(\hat x, \hat y')^T\hat \eta'' + \hat \xi)+i y' \cdot( b_{y'}(\hat x, \hat y')^T\hat\eta '' + \hat \eta')} = \frac{(1+ \Delta_{x,y', \zeta''})^\ell e^{i(b(\hat x,\hat y') - \hat y '')\cdot \zeta'' +i x \cdot (b_x(\hat x, \hat y')^T\hat \eta'' + \hat \xi)+i y' \cdot( b_{y'}(\hat x, \hat y')^T\hat\eta '' + \hat \eta')}}{\llangle  \left(b(\hat x,\hat y'), b_x(\hat x,\hat y')^T \hat \eta'', b_{\hat y}(\hat x,\hat y')^T \hat \eta''\right) - \left(\hat y'', -\hat \xi, -\hat \eta'\right) \rrangle^\ell}.
\end{equation*}
}
After inserting this expression into \eqref{eq:taylor-expanded-form}, integration by parts (see the proof of \cite[Theorem~3.1]{cordero2007time} for reference) yields the desired estimate.
\end{proof}

\Cref{lem: b bound for integral} indicates that $R$ is well-organized when represented using Gabor atoms.
For any $r>1$, we want to show that for each Gabor atom $h_{\hat y,\hat \eta}$ on $Y$, we have the following bound on the order statistics of $R^*h_{\hat y,\hat \eta}$ when represented using the Gabor system on $X$:
\begin{equation}\label{eq:desired-powerlaw-decay}
    \big|\langle R^*h_{\hat y,\hat \eta}, g_{m,n}\rangle_{(\hat x, \hat \xi)\in\Lambda}\big|_{(k)}
    \lesssim k^{-r},\quad k\geq 1.
\end{equation}
Equivalently, for some suitable $C>0$ independent of $(\hat y,\hat\eta)$,
\begin{equation}\label{eq:weak-sequence-norm}
    \#\{(\hat x, \hat \xi)\in\Lambda:
    |\langle R^*h_{\hat y,\hat \eta}, g_{\hat x, \hat \xi}\rangle| > Ck^{-r}\} \leq k,
    \quad k\geq 1.
\end{equation}
\begin{lemma}\label{lem: weak sequence norm holds}
The estimate \eqref{eq:weak-sequence-norm} holds.
\end{lemma}
\begin{proof}
First, observe that by changing the constant $C>0$ in \eqref{eq:weak-sequence-norm}, it suffices to prove
\begin{equation}\label{eq:weak-sequence-norm-2}
    \#\{(\hat x, \hat \xi)\in\Lambda:
    |\langle R^*h_{\hat y,\hat \eta}, g_{\hat x, \hat \xi}\rangle| > Ck^{-r}\} \leq C' k,
    \quad k\geq 1.
\end{equation}
for some $C'$ independent of $(\hat y, \hat \eta)$.

By choosing $\ell = \lceil n r\rceil$ in \cref{lem: b bound for integral}, we have the following inclusion:
{\footnotesize
\begin{eqnarray}
&&\{(\hat x, \hat \xi)\in\Lambda:
    |\langle R^*h_{\hat y,\hat \eta}, g_{\hat x, \hat \xi}\rangle| > Ck^{-r}\} \subset \\ 
    &&\nonumber \left\{(\hat x, \hat \xi)\in \Lambda:  
    (p\circ q^{-1})(\supp(g_{\hat x,0}))\cap\supp(h_{\hat y,0})\neq\emptyset,
    \|  \left(b(\hat x,\hat y'), b_x(\hat x,\hat y')^\top \hat \eta'', b_{ y'}(\hat x,\hat y')^\top \hat \eta''\right) - \left(\hat y'', -\hat\xi, -\hat \eta'\right) \|^2 
    < \left(\frac{C_\ell}{C}\right)^{2/\ell} k^{2r/\ell}\right\}
\end{eqnarray}
}
Therefore, to show \eqref{eq:weak-sequence-norm-2}, it suffices to show that
{\scriptsize
\begin{eqnarray}\label{eq: upper bound on set}
\#\left\{(\hat x,\hat\xi)\in \Lambda:  (p\circ q^{-1})(\supp(g_{\hat x,0}))\cap\supp(h_{\hat y,0})\neq\emptyset,  
\|  \left(b(\hat x,\hat y'), b_x(\hat x,\hat y')^\top \hat \eta'', b_{ y'}(\hat x,\hat y')^\top \hat \eta''\right) - \left(\hat y'', -\hat\xi, -\hat \eta'\right) \|^2 
    < \left(\frac{C_\ell}{C}\right)^{2/\ell} k^{2r/\ell}\right\}\leq C'k
\end{eqnarray}
}
for $k \in \mathbb N$.
We will now establish that choosing $C\approx C_\ell, \ell=\lceil n r\rceil$ yields~\eqref{eq: upper bound on set}.
For $(y,\eta)\in T^*Y$ with $\|\eta\|\leq 1$ and $s>0$, define the following sets:
\begin{gather*}
    A_1(y) =
    \{ x\in F(\Lambda)) : 
    \supp(Rg_{x,0})\cap\supp(h_{y,0})\neq\emptyset \}, \\
    A_2(y,\eta;s) =
    \{ x\in A_1(y) : 
    \| b_{y'}(x,y')^\top \eta'' + \eta' \| < s \},
\end{gather*}
%
where $F:T^*X\to X$ is the map that ``forgets'' the cotangent vector.
Assume that the window function generating the Gabor system on $X$ has sufficiently small support.
By the compactness of the fibers $G_y$, $\# A_1(y)$ is uniformly bounded over $y\in Y$.
Since $A_2(y,\eta;s)\subset A_1(y)$, for all such $(y,\eta)$ and $s>0$, we have
\begin{equation}\label{eq:cone-count-one}
    \# A_2(y,\eta;s) \leq C_0
\end{equation}
for some constant $C_0>0$ independent of $(y,\eta)$ and $s>0$.

Now let $(\hat y, \hat \eta) \in T^*Y \setminus 0$ and $k\in \mathbb N$.
Then, by definition of $A_2(y,\eta; s)$, we have that 

{\scriptsize
\begin{eqnarray}
\nonumber && \left\{(\hat x,\hat\xi)\in \Lambda:  (p\circ q^{-1})(\supp(g_{\hat x,0}))\cap\supp(h_{\hat y,0})\neq\emptyset,  \|  \left(b(\hat x,\hat y'), b_x(\hat x,\hat y')^\top \hat \eta'', b_{ y'}(\hat x,\hat y')^\top \hat \eta''\right) - \left(\hat y'', -\hat\xi, -\hat \eta'\right) \|^2 
    < 
    \left(\frac{C_\ell}{C}\right)^{2/\ell} k^{2r/\ell}\right\}\\
    &&\subset\left \{ (\hat x,\hat\xi)\in \Lambda : \hat x\in A_2\left(\hat y, \frac{\hat \eta}{|\hat \eta|}; \frac{k^{1/n}}{|\hat\eta|}\right), \left\|b_x(\hat x,\hat y')^\top \frac{\hat \eta''}{|\hat \eta|} + \frac{\hat\xi}{|\hat \eta|}\right\|< \frac{k^{1/n}}{|\hat\eta|}\right\}.
\end{eqnarray}
}
So to show \eqref{eq: upper bound on set}, it suffices to show
\begin{eqnarray}
    \label{eq: rougher upper bound}
    \#\left \{ (\hat x,\hat\xi)\in \Lambda : \hat x\in A_2\left(\hat y, \frac{\hat \eta}{|\hat \eta|}; \frac{k^{1/n}}{|\hat\eta|}\right), \left\|b_x(\hat x,\hat y')^\top {\hat \eta''} + {\hat\xi}\right\|< {k^{1/n}}\right\}<C'k
\end{eqnarray}
Note that there are only finitely many $\hat x\in A_2(y,\hat\eta/|\hat \eta|;k^{1/n}/|\hat \eta|)$ by \eqref{eq:cone-count-one}. 
Additionally, for each such $\hat x\in A_2(y,\hat\eta/|\hat \eta|;k^{1/n}/|\hat \eta|)$, we have that
\begin{eqnarray}
\{\hat \xi \mid \left\| b_x(\hat x, \hat y')^\top \hat \eta'' + \hat \xi\right\| < k^{1/n}\} < Ck    
\end{eqnarray}
with the constant $C'>0$ independent of $\hat x$, $\hat y$, and $\hat \eta$. 
This is due to the fact that we are simply counting the number of $\hat \xi$ within a radius of $k^{1/n}$ of $b_x(\hat x, \hat y')^\top \hat \eta''$, which is independent of $b_x(\hat x, \hat y')^\top \hat \eta''$ when we use a rectangular lattice.
We have therefore that \eqref{eq: rougher upper bound} holds and the proof is complete.
\end{proof}

\subsection{Compressive Recovery}\label{sec:recovery:compressive}

The compressibility result above holds uniformly for $(\hat y,\hat \eta)\in T^*Y$: that is, we do not need to restrict it to a lattice contained in $T^*Y$.
Thus, for any given Gabor atom $h_{\hat y,\hat \eta}$ on $Y$, we can recover $R^*h_{\hat y,\hat \eta}$ up to some resolution from only a few inner products $\{\langle R^*h_{\hat y,\hat \eta}, u_j\rangle\}_{j=1}^J$, where each $u_j$ is a linear combination of Gabor atoms on $X$.
Let $\Omega\subset\Lambda$ be a finite subset of the lattice indexing the Gabor frame for functions on $X$, and consider the space of functions $\vspan\{g_{\hat x,\hat\xi}\}_{(\hat x,\hat\xi)\in\Omega}$ with norm
\begin{equation*}
    \|u\|^2_{\Omega} := 
    \sum_{\omega\in\Omega}|\langle u, \gamma_\omega\rangle|^2,
\end{equation*}
that is, the $\ell^2(\Omega)$ norm of $u$ as a weighted sum of Gabor atoms indexed by $\Omega$, with coefficients computed using the dual frame $\{\gamma_{\omega}\}_{\omega\in\Omega}$.
Call this space $L^2_\Omega(X)$, noting that it is a finite-dimensional subset of $L^2(X)$.

We prove the following result on approximating $R$ over $L^2_{\Omega}(X)$ as a precursor to proving the data-efficiency theorem.
\begin{proposition}\label{prop:pre-data-efficiency}
    Let $R:\mathcal{D}'(X)\to\mathcal{D}'(Y)$ be a double fibration transform satisfying the Bolker condition, and let $\epsilon>0$ be sufficiently small.
    Then, for all $r>1$, there exists a randomized algorithm yielding an estimate $\widehat{R}$ of $R$ from $J=(\frac{C_r}{\epsilon})^{1/r}\log|\Omega|$ input-output examples $\{u_j,Ru_j\}_{j=1}^J$ such that
    \begin{equation*}
        \big|\langle h_{\hat y,\hat\eta}, (\widehat{R}-R)u\rangle\big|<\epsilon\|u\|_{L^2(\Omega)}
    \end{equation*}
    simultaneously for all $(\hat y,\hat\eta)\in T^*Y$ and $u\in L_{\Omega}^2(X)$ with high probability.
\end{proposition}

\begin{proof}
    It will be useful to rewrite the guarantee of the proposition in terms of a matrix representation of the operators $\widehat{R}$ and $R$.
    Observe that the norm bound can be established by proving that
    \begin{equation*}
        \sup_{(\hat y,\hat \eta)\in T^*Y, u\in L^2_\Omega(X),\|u\|_{\Omega}=1}
        \big| \langle h_{\hat y,\hat \eta},  (R-\widehat{R})u \rangle \big| 
        < \epsilon.
    \end{equation*}
    To establish this, we construct an estimate $\widehat{R}^*h_{\hat y,\hat\eta}$ such that
    \begin{equation}\label{eq:epsilon-norm-bound}
        \sup_{(\hat y,\hat\eta)\in T^*Y}
        \sqrt{\sum_{\omega\in\Omega}\langle (R-\widehat{R})^* h_{\hat y,\hat\eta}, g_\omega \rangle^2}
        < \epsilon
    \end{equation}
    as follows.
    For each $j=1,\ldots,J$, put
    \begin{equation*}
        u_j = \sum_{\omega\in\Omega} a_{j\omega}g_\omega,
    \end{equation*}
    where $\{a_{j\omega}\}_{j,\omega}$ are \textit{i.i.d.} normally distributed random variables.
    Then, for any $(\hat y,\hat\eta)\in T^*Y$, consider the corresponding measurements
    \begin{equation*}
        c_j := \langle h_{\hat y,\hat\eta}, Ru_j\rangle = \langle R^*h_{\hat y,\hat\eta}, u_j\rangle.
    \end{equation*}
    Observe that
    \begin{equation*}
        c_j = \sum_{\omega\in\Omega}a_{j\omega}
        \langle R^*h_{\hat y,\hat\eta},g_\omega\rangle.
    \end{equation*}
    Denoting the coefficients $\langle R^*h_{\hat y,\hat\eta},g_\omega\rangle$ by $v_{\omega}$, and gathering the coefficients $\{a_{j\omega}\}_{j,\omega}$  in a matrix $A\in\mathbb{R}^{J\times\Omega}$, this data can be written as $c=Av$, where $v$ is the vector we wish to recover.
    By \cref{lem: weak sequence norm holds}, we have $|v|_{(k)}\leq C_r k^{-r}$, where $C_r>0$ depends on $r$, but not on $(\hat y, \hat \eta)$.
    Because of this, a sparse approximation of $v$ by keeping the $S=\lceil (\epsilon/C_r)^{1/(1/2-r)} \rceil$ largest entries has an $\ell_1$ error at most $\epsilon$.
    By standard results in compressive sensing~\cite{candes2006stable}, $A$ is such that the vector $v$ can be recovered with $\ell_2$ error at most $\epsilon$ via $\ell_1$-minimization when $J=O((\epsilon/C_r)^{1/(1/2-r)}\log|\Omega|)$, with high probability, where the constants obscured by the $O(\cdot)$ notation are independent of $r$.
    Note that this condition holds simultaneously \emph{for all} $(\hat y,\hat\eta)\in T^*Y$.
    This fulfills the condition \eqref{eq:epsilon-norm-bound}, completing the proof.
\end{proof}
We now show how \cref{prop:pre-data-efficiency} implies \cref{theorem:data-efficiency}:
\begin{proof}[Proof of \cref{theorem:data-efficiency}]
Suppose $\Omega=\{(\hat x, \hat \xi)\in\Lambda:|\hat\xi|\leq B\}$ for some $B>0$.
Let $u\in L^2(X)$ be given such that $\|u\|_{L^2(X)}=1$.
Then, $u$ can be written as
\begin{align*}
    u &= \sum_{(\hat x, \hat \xi)\in\Omega}\langle u, \gamma_{\hat x,\hat\xi}\rangle g_{\hat x,\hat\xi}
    +
    \sum_{(\hat x, \hat \xi)\in\Lambda\setminus\Omega}\langle u, \gamma_{\hat x,\hat\xi}\rangle g_{\hat x,\hat\xi} \\
    &:= u_\Omega + u_{\Omega^c},
\end{align*}
where $\|u_\Omega\|_{\Omega}\leq C_\gamma$, where $C_\gamma$ is the upper frame bound of the Gabor system $\{\gamma_{\hat x,\hat \xi}\}_{(\hat x,\hat\xi)\in\Lambda}$.
Based on this decomposition, we extend the estimated operator $\widehat{R}$ from \cref{theorem:data-efficiency} to act on general $u\in L^2(X)$ as
\begin{equation*}
    \widehat{R}u := \widehat{R}u_{\Omega}.
\end{equation*}
We now bound $\|Ru_{\Omega^c}\|$.
Observe that
\begin{equation*}
\|Ru_{\Omega^c}\| \leq \sqrt{\|u_{\Omega^c}\|\cdot \|R^*Ru_{\Omega^c}\|},
\end{equation*} 
and see further that
\begin{equation}\label{eq:psido-bound-1}
    \|R^*Ru_{\Omega^c}\|^2
    \leq \sum_{(\hat x, \hat\xi)\in\Lambda_X\setminus\Omega}
    |\langle u, \gamma_{\hat x,\hat\xi}\rangle|^2 \|R^*Rg_{\hat x,\hat\xi}\|^2.
\end{equation}
By \cite[Theorem~2.2]{mazzucchelli2023} and the Bolker condition, $R^*R$ is a pseudodifferential operator of order $s:=n''-(n_X+N_Y)/2<0$.
Since $g\in C_0^{\infty}(\mathbb{R}^{n_X})$, its Fourier transform has superalgebraic decay away from zero, yielding the following estimate for all $(\hat x,\hat\xi)\in T^*\mathbb{R}^{n}$:
\begin{equation*}
    \|R^*Rg_{\hat x,\hat\xi}\| \leq C_1 (1+\|\hat\xi\|)^s \|g\|
\end{equation*}
for some $C_1>0$.
Since $(\hat x,\hat\xi)\in\Lambda\setminus\Omega$ implies $\|\hat\xi\|>B$, this implies that
\begin{equation}\label{eq:psido-bound-2}
    \|R^*Rg_{\hat x, \hat\xi}\| \leq C_1(1+B)^s
\end{equation}
for all $(\hat x,\hat\xi)\in\Lambda_X\setminus\Omega$.
Substituting \eqref{eq:psido-bound-2} into \eqref{eq:psido-bound-1} yields
\begin{equation}\label{eq:psido-bound-3}
    \|R^*Ru_{\Omega^c}\| \leq C_\gamma^2 C_1(1+B)^s.
\end{equation}
This allows us to finally bound the error when tested against Gabor atoms for general $u\in L^2(X)$.
Taking $\|u\|_{L^2(X)}=1$ without loss of generality, under the conditions of \cref{theorem:data-efficiency}, we have
\begin{align*}
    |\langle h_{\hat y,\hat\eta}, (\widehat{R}-R)u\rangle|
    &=
    |\langle h_{\hat y,\hat\eta}, (\widehat{R}-R)u_{\Omega}-Ru_{\Omega^c}\rangle| \\
    &\leq C_\gamma \epsilon + \|Ru_{\Omega^c}\|_{L^2(Y)} \\
    &\leq C_\gamma \epsilon + C_\gamma^2 C_1 (1+B)^s.
\end{align*}
In other words, the error is $O(\epsilon + (1+B)^s)$.
We choose $B\approx\epsilon^{1/s}$, so that $|\Omega|=O(\epsilon^{n/s})$.
Then, \cref{prop:pre-data-efficiency} yields the desired sampling rate, on account of $s<0$.
\end{proof}

\begin{remark}
    We remark that the constant $C_r$ in \cref{theorem:data-efficiency} is proportional to the constant in the numerator of the decay bound proved in \cref{lem: b bound for integral}.
    Because of this, tighter control of this constant can lead to more accurate sample bounds: for instance, if $C_r=O(r!)$, then \cref{theorem:data-efficiency} implies a $\polylog(1/\epsilon)$ sampling rate.
    To attain such control over $C_r$ requires much stronger assumptions on the operator, 
    such as a (Gevrey-)analytic structure~\cite{cordero2015exponentially}.
\end{remark}

\section{Proof of Lemma~\ref{lemma:gauss-approximation}}

\begin{proof}
    We begin by establishing \cref{lemma:gauss:one}.
    Fix $y\in Y$, so that we may treat $f(y,x)$ as a function only of $x$, which we will write as $f(x)$, abusing notation.
    Let $x\in f^{-1}(0)$ be given arbitrarily; such a point is guaranteed to exist by \cref{assump:nontrivial-levelset}.
    By the constant rank theorem and \cref{assump:levelset-gradient}, there is a neighborhood $U\ni x$ such that $f^{-1}(0) \cap U$ is a submanifold of $X$ with codimension $n''$.
    Since this holds for all $x\in f^{-1}(0)$, that set constitutes a submanifold of $X$ with codimension $n''$.
    Moreover, \eqref{eq:implicit-measure} defines a smooth measure on $Z_f$.
    
    To establish \cref{lemma:gauss:two}, we again fix $y\in Y$ and treat $f(y,x)$ as $f(x)$.
    Assume without loss of generality that $w(y)=1$.
    Since $f^{-1}(0)$ is a submanifold of $X$, for any $x\in f^{-1}(0)$, by the constant rank theorem we may take coordinates $(x',x'')$ locally write $f^{-1}(0)$ as the graph of a function $x''=b(x')$.
    Furthermore, we approximate $f(x',x'')$ near $f^{-1}(0)$ by the first-order term $\nabla_{x''} f(x',b(x'))\cdot(x''-b(x'))$.
    This yields the estimate
    \begin{equation*}
        \left(\frac{\lambda}{\pi}\right)^{n''/2}
        \int e^{-\lambda|f(x',x'')|^2} dx''
        \approx a_f(x',b(x')),
    \end{equation*}
    or, in other words,
    \begin{equation*}
        \left(\frac{\lambda}{\pi}\right)^{n''/2}
        \int e^{-\lambda|f(x',x'')|^2} dx' dx''
        \approx \int_{G_y} a_f(y,x) d\nu_y(x).
    \end{equation*}
    In local coordinates, then, the error can be written as
    \begin{equation}\label{eq:gauss-error-one}
    Lu(y) - L^{\lambda}u(y)
    \approx
    \iint
    \frac{(\lambda/\pi)^{n''/2}e^{-\lambda|f(x',x'')|^2}u(x',x'') - a_f(y,x',b(x'))u(x',b(x'))}
    {\int_{G_y} a_f(y,s) d\nu_y(s)}
    dx''dx'.
    \end{equation}
    For each $x'$, we have the first order approximation
    \begin{align*}
        &\big|
        \int 
        \frac
        {
        \left(\frac{\lambda}{\pi}\right)^{n''/2}
        e^{-\lambda|\nabla_{x''}f(x',b(x'))\cdot(x''-b(x'))|^2}
        }{a(y,(x',b(x')))}
        \left(u(x',x'')-u(x',b(x'))\right)
        dx''
        \big| \\
        &\qquad\qquad\qquad\qquad\leq
        \sup_{u''}|\partial_{x''}u(x',u'')|
        \int 
        \frac
        {
        \left(\frac{\lambda}{\pi}\right)^{n''/2}
        e^{-\lambda|\nabla_{x''}f(x',b(x'))\cdot x''|^2}
        }{a(y,(x',b(x')))}
        \big|x''|
        dx'' \\
        &\qquad\qquad\qquad\qquad\leq
        C_f \lambda^{-n''/2} \sup_{u''}|\partial_{x''}u(x',u'')|
    \end{align*}
    for some $C_f>0$ depending on $f$, but not $u,x',$ or $\lambda$.
    Substituting into \eqref{eq:gauss-error-one} yields the bound
    \begin{equation*}
    |Lu(y) - L^{\lambda}u(y)|
    \lesssim
    \lambda^{-n''/2} \sup_{x''}|\partial_{x''}u(x',x'')|,
    \end{equation*}
    which implies the desired bound.
\end{proof}

\section{Proof of Lemma~\ref{lemma:universal}}

\begin{proof}
    Let $P_Y,P_X$ be the collection of all monomials on $Y,X,$ respectively.
    For positive integers $\ell,k>0$, the quantity
    \begin{equation}\label{eq: N asympt}
        F(\ell, k) := \sum_{j=0}^\ell \binom{k + j- 1}{j} \sim \frac{\ell^{k}}{(k-1)!}
    \end{equation}
    counts the number of possible monomials in $k$ variables of degree less than $\ell$.
    We assume without loss of generality that $m=1$.
    Note that any polynomial of degree $\ell$ on $Y\times X$ can be written in the form $\langle p_Y, Wp_X\rangle$ for $p_Y\in P_Y^{F(\ell,N)}, p_X\in P_X^{F(\ell,n)}, W\in\mathbb{R}^{F(\ell,N)\times F(\ell,n)}$.
    Then, by \cite[Theorem~1]{bagby2002multivariate}, there exists such $p_Y,p_X,W$ with
    \begin{equation*}
        \| f - \langle p_Y,W p_X\rangle\|_{\infty} \leq \frac{C\|f\|_{C^2}}{\ell^2},
    \end{equation*}
    where $C$ depends only on the dimension and diameter of $Y\times X$ in Euclidean space.
    Since $\Psi_Y,\Psi_X$ are dense subsets of $C^1(Y),C^1(X)$, we can choose $\psi_Y\in\Phi_Y^{F(\ell,N)},\psi_X\in\Phi_X^{F(\ell,n)}$ such that
    \begin{equation*}
        \| \langle \psi_Y , \psi_X\rangle - \langle p_Y,W p_X\rangle\|_{\infty} \leq \frac{C\|f\|_{C^2}}{\ell^2},
    \end{equation*}
    on account of $P_Y,P_X$ consisting of smooth functions.
    Choosing $\ell\geq \sqrt{2C\|f\|_{C^2}/\epsilon}$ and $d=F(\ell,\max\{N,n\})$, applying the triangle inequality yields
    \begin{equation*}
        \|f - \langle\psi_Y(y),\psi_X(x)\rangle\|_{\infty}\leq\epsilon,
    \end{equation*}
    with $d=O((\epsilon/\|f\|_{C^2})^{-\max\{N,n\}/2})$.

    Put $\hat{f}(y,x)=\langle\psi_Y(y),\psi_X(x)\rangle$, and choose $\hat{w}(y)$ so that
    \begin{equation}\label{eq:kernel-normalization}
        \frac{\hat{w}(y)}{\int_X e^{-\lambda|\hat{f}(y,x)|^2}d\nu(x)}
        =
        \frac{{w}(y)}{\int_X e^{-\lambda|{f}(y,x)|^2}d\nu(x)}
        \approx \lambda^{n''/2},
    \end{equation}
    where the approximation follows from \cref{assump:levelset-gradient}.
    We may assume without loss of generality that $\nu(X)=1$ (by absorbing constants into the final bound), so that
    \begin{equation*}
        \left|\int_X 
        \left(
        e^{-\lambda|f(y,x)|^2}-e^{-\lambda|\hat{f}(y,x)|^2}
        \right)
        u(x)
        d\nu(x)\right|
        \leq
        \sqrt{\lambda}\epsilon\|u\|_{\infty},
    \end{equation*}
    since $\exp(-x^2)$ has Lipschitz constant bounded by one.
    Incorporating the normalization~\eqref{eq:kernel-normalization} and appealing to the definition of the levelset integral kernel, this concludes the proof.
\end{proof}

\section{Proof of Lemma~\ref{lemma:discrete-kernel}}

\begin{proof}
    Define
    \begin{equation*}
        c_{\lambda,0} = \int_X e^{-\lambda|f(y,x)|^2}d\nu_0(x).
    \end{equation*}
    Observe that
    \begin{align*}
        |c_\lambda-c_{\lambda,0}| &= \big|\int_X e^{-\lambda|f(y,x)|^2}d(\nu-\nu_0)(x)\big| \\
        &\leq \Lip\left(e^{-\lambda|f(y,x)|^2}\right)\cdot W_1(\nu,\nu_0) \\
        &\leq \sqrt{\lambda} K_f\cdot W_1(\nu,\nu_0).
    \end{align*}
    In particular, $|c_{\lambda}-c_{\lambda,0}|\leq c_{\lambda}/2$.
    Next, define
    \begin{align*}
        \alpha &= \int_X e^{-\lambda|f(y,x)|^2}u(x)d\nu(x) \\
        \alpha_0 &= \int_X e^{-\lambda|f(y,x)|^2}u(x)d\nu_0(x).
    \end{align*}
    Observe that $e^{-\lambda|f(y,x)|^2}u(x)$ has $L^{\infty}$-norm bounded by one and Lipschitz constant at most $\sqrt{\lambda} K_f + K_u$.
    This implies that $|\alpha|\leq 1$ and
    \begin{equation*}
        |\alpha-\alpha_0| \leq (\sqrt{\lambda} K_f + K_u)\cdot W_1(\nu,\nu_0).
    \end{equation*}
    Finally, noting that $0<c\leq 1$, we have
    \begin{align*}
        |L^{\lambda}_{\nu}u(y) - L^{\lambda}_{\nu_0}u(y)|
        &= \big|\frac{\alpha}{c_\lambda}-\frac{\alpha_0}{c_{\lambda,0}}\big| \\
        &\leq \frac{|\alpha|\cdot|c_\lambda-c_{\lambda,0}|+|c_\lambda|\cdot|\alpha-\alpha_0|}{c_\lambda^2-|c_\lambda|\cdot|c_\lambda-c_{\lambda,0}|} \\
        &\leq \frac{2\sqrt{\lambda}K_f+K_u}{c_\lambda^2/2}W_1(\nu,\nu_0),
    \end{align*}
    as desired.
\end{proof}

\section{Description of Models and Training}

For all experiments, the maps $\psi_Y,\psi_X$ were parameterized as multilayer perceptrons (MLPs) with sinusoidal activations~\cite{sitzmann2020implicit}, and three hidden layers of width $d=64$, followed by linear layers yielding maps $\psi_Y:Y\to\mathbb{R}^{64\times m}, \psi_X:X\to\mathbb{R}^{64\times m}$, where $m$ varies depending on the experiment.

Given a sampled dataset $\{S_1u_j, S_2Ru_j+\eta_j\}_{j=1}^J$, where $\eta_j$ models additive noise, models were trained using the Adam optimizer~\cite{kingma2015adam,deepmind2020jax} to minimize the MSE defined as
\begin{equation*}
    \mathrm{MSE}(\theta;L,\lambda,\{S_1u_{j_k}, S_2Ru_{j_k}+\eta_{j_k}\}_{k=1}^{K})
    =
    \frac{1}{K}\sum_{k=1}^K \| S_2Ru_{j_k}+\eta_{j_k} - S_2L_\theta^\lambda S_1u_{j_k} \|^2,
\end{equation*}
where the indices $\{j_k\}_k$ correspond to the batch at a given gradient step, $\theta$ gathers the parameters of the neural networks, and $L_{\theta}^\lambda S_1u_{j_k}$ is understood to operate according to the discretization scheme described in the paper.

In all experimental runs, for a given training dataset, we reserve $20\%$ of the data as a validation set, to choose the model that performs best.
The remaining $80\%$ are used to train $5$ models with different random initializations each for $10000$ gradient steps, and the one that performs best on the validation set is trained for $40000$ additional gradient steps.
The number of training steps used was double for the lensing ray transform, in order to attain more accurate models with which to estimate the wavespeed and subsequent geodesics.

The set of model parameters with the best performance on the validation set over the whole training run is kept.
A batch size of $K=32$ was used for the experiments learning the Radon transform in $\mathbb{R}^2$ and the Euclidean Ray transform in $\mathbb{R}^3$, and a batch size of $K=256$ for the lensing ray transform example.

\bibliographystyle{abbrv}
\bibliography{strings,ref}



\end{document}